\documentclass{article}

\usepackage{microtype}
\usepackage{graphicx}
\usepackage{subfigure}
\usepackage{booktabs} 
\usepackage{float} 
\usepackage{svg} 

\usepackage{hyperref}

\usepackage[accepted]{icml2024}

\usepackage{amsmath}
\usepackage{amssymb}
\usepackage{mathtools}
\usepackage{amsthm}

\usepackage[capitalize,noabbrev]{cleveref}

\usepackage{stackengine,tikz,graphicx}
\newcommand\solidcirc[4][0]{\rotatebox{#1}{\tikz{\draw[line width=#2] (0,0) 
			arc [x radius=#3,y radius=#4,start angle=0,end angle=360];}}}

\theoremstyle{plain}
\newtheorem{theorem}{Theorem}[section]

\theoremstyle{definition}

\theoremstyle{remark}

\usepackage[textsize=tiny]{todonotes}

\icmltitlerunning{$\boldsymbol{\mu}$-Net: ConvNext-Based U-Nets for Cosmic Muon Tomography}

\begin{document}
	
\twocolumn[
\icmltitle{$\boldsymbol{\mu}$-Net: ConvNext-Based U-Nets for Cosmic Muon Tomography}



\icmlsetsymbol{equal}{*}

\begin{icmlauthorlist}
	\icmlauthor{Lim Li Xin Jed}{equal,sch}
	\icmlauthor{Qiu Ziming}{equal,sch}
\end{icmlauthorlist}

\icmlaffiliation{sch}{Graduate of NUS High School of Mathematics and Science, Singapore}

\icmlcorrespondingauthor{Lim Li Xin Jed}{jedlimlx@outlook.com}
\icmlcorrespondingauthor{Qiu Ziming}{qiuzeeming@gmail.com}

\icmlkeywords{Machine Learning, Deep Learning, U-Net, ConvNeXt, Muon Tomography}

\vskip 0.3in

]



\printAffiliationsAndNotice{Our model code is at \url{https://github.com/jedlimlx/Muon-Tomography-AI} and our data generation code is at \url{https://github.com/jedlimlx/Muons-Data-Generation}. Our dataset can be found at \url{https://www.kaggle.com/datasets/tomandjerry2005/muons-scattering-dataset}. \icmlEqualContribution} 

\begin{abstract}
	Muon scattering tomography utilises muons, typically originating from cosmic rays to image the interiors of dense objects. However, due to the low flux of cosmic ray muons at sea-level and the highly complex interactions that muons display when travelling through matter, existing reconstruction algorithms often suffer from low resolution and high noise. In this work, we develop a novel two-stage deep learning algorithm, $\mu$-Net, consisting of an MLP to predict the muon trajectory and a ConvNeXt-based U-Net to convert the scattering points into voxels. $\mu$-Net achieves a state-of-the-art performance of 17.14 PSNR at the dosage of 1024 muons, outperforming traditional reconstruction algorithms such as the point of closest approach algorithm and maximum likelihood and expectation maximisation algorithm. Furthermore, we find that our method is robust to various corruptions such as inaccuracies in the muon momentum or a limited detector resolution. We also generate and publicly release the first large-scale dataset that maps muon detections to voxels. We hope that our research will spark further investigations into the potential of deep learning to revolutionise this field.\\
\end{abstract}

\section{Introduction}

Muon tomography is an imaging technique that utilises muons, typically originating from cosmic rays, to image the interiors of objects. By leveraging the fact that muons are highly penetrating particles, muon tomography offers a non-invasive and non-destructive means of investigating the internal composition of dense materials. Through tracking and analysis of muon trajectories and energies, this enables the accurate reconstruction of internal structures and features. Consequently, muon tomography has emerged as a tool with wide-ranging applications in fields such as geophysics \cite{volcano}, civil engineering \cite{civil_engineering}, and archaeology \cite{mines, giza}. \\


Due to cosmic rays entering the Earth's atmosphere, there is no need for a specialised muon source unlike other types of tomography. However, there are many other challenges in this task. First, the flux of cosmic muons at sea-level is very low \cite{muon_flux}. In order to produce a decent reconstruction, data has to be collected for a long period of time. Furthermore, unlike other types of tomography such as x-ray computed tomography, muons will scatter off atomic nuclei. This makes the forward operator highly nonlinear. In contrast, in computed tomography, x-rays only attenuate and hence, the forward operator is the linear radon transform. In addition, due to limitations of current day muon detectors, the momentum of muons which significantly affects the muon scattering angle is not known. At best, we will have an estimate of the momentum $\hat{p}$ with a significant amount of uncertainty. \\

Several methods have been developed to tackle this problem. The first algorithm developed was the Point of Closest Approach (PoCA) algorithm \cite{poca} which assumes that the muons only scatter once at the point of closest approach of its inward and outward trajectory. The Maximum Likelihood and Expectation Maximisation (MLEM) algorithm \cite{mlem} improves on PoCA and iteratively optimises the reconstruction to maximise the likelihood of producing a given scattering outcome. Other algorithms have also been developed such as maximum a posterori (MAP) \cite{map}, most likely path \cite{mlp1,mlp2,mlp3}, scattering density estimation (SDE) \cite{sde} and angle statistics reconstruction (ASR) \cite{asr}. Other methods have also been developed for low dosages such as the binned clustering algorithm \cite{binned} and a method based on density clustering \cite{density-clustering}. \\

However, no one has attempted to make use of developments in deep learning to directly solve this ill-posed problem and go directly from muon detections to a 3D reconstruction. Due to the abundance of data which can be obtained from simulations from software such as Geant4 \cite{geant4} and the non-linearity of the problem, deep learning methods are well-suited for this task. They have been previously applied to other inverse problems such as computed tomography \cite{ct}. \\

In this work, we develop a novel two-stage deep learning algorithm, $\mu$-Net, for cosmic muon scattering tomography, based on using the point of closest approach (PoCA) algorithm and the U-Net architecture proposed by \citet{unet}. To our knowledge, this is the first application of deep learning in muon scattering tomography to directly perform the 3D reconstruction. Instead of using the more traditional Residual Blocks, we make use of ConvNeXt Blocks \cite{convnext}. 
We find that our model significantly outperforms traditional reconstruction algorithms (in speed and accuracy) and achieves state-of-the-art performance. Our method achieves a PSNR of 17.14 with only 1024 muons while PoCA achieves a PSNR of only 13.66 while taking 22.5s to run while $\mu$-Net runs in 126 ms.

We have also generated the first large-scale dataset mapping muon detections to voxels. In prior literature, there has been no standard benchmark to evaluate various methods for 3D muon tomographic reconstructions and no quantitative metrics used to compare different methods. As such, we release our dataset and data generation code publicly. We hope this will spark more systematic investigations into reconstruction algorithms on this task, using deep learning or using traditional methods.\\

\section{Preliminaries}

\subsection{Physical Background}

Muon scattering relies primarily on modelling the scattering interaction between muons and matter. \citet{mcs} developed a scattering theory for charged particles and found that charged particles travelling through a plate of thickness $x$ that undergo Coulomb scattering have scattering angles and lateral displacements that follow a Gaussian distribution with mean, $\mu=0$ and variance,

\begin{equation}
	\sigma^2=\frac{E_s^2}{2p^2v^2}\frac{x}{L_{rad}}
\end{equation}

where $E_s=21$ MeV, $\lambda$ is the radiation length of the material, $x$ is the thickness of the plate and $p$ and $v$ are the momentum and velocity of the muon respectively.

From this formula, we can define the parameter of interest, the scattering density $\lambda$.

\begin{equation}
	\lambda=\frac{\sigma^2}{x}=\frac{15 MeV}{p^2v^2}\frac{1}{L_{rad}}
\end{equation}

With this, the task of muon tomography is to find the distribution of $\lambda$ within the object through looking at positions and directions of the incoming and outgoing muons.

\subsection{Problem Statement}

Before proceeding with a literature review and the description of the method, let us formalize the muon reconstruction problem. We shall follow a similar notation to \citet{mlem}.

Let the object of interest be defined by its scattering density,

\begin{equation}
	\lambda(x,y,z)=\left(\frac{15}{p_0}\right)^2\frac{1}{L_{rad}(x,y,z)}
\end{equation}

where $L_{rad}(x,y,z)$ is the radiation length at each point within the object.

We can represent the scattering density in terms of some basis functions $\phi_j(x,y,z)$ such that

\begin{equation}
	\lambda(x,y,z)=\sum_j\alpha_j\phi(x,y,z)
\end{equation}

where $\boldsymbol{\alpha}=\left[\alpha_j\right]$ are the coefficients for the basis functions.

Suppose the muon detections, $\mathbf{y}$ follow a distribution $D$ parameterized by the scattering density of the object, i.e.

\begin{equation}
	\mathbf{Y} \sim D(\boldsymbol{\alpha})
\end{equation}

Therefore, given a sample of $n$ muon detections $Y_n=\left[\mathbf{y_1}, \mathbf{y_2}, \dots, \mathbf{y_n}\right]$ we wish to construct a point estimate $\mathbf{a}(Y_n)$ of $\boldsymbol{\alpha}$, which is approximated by the deep neural network $\boldsymbol{f_\theta}(Y_n)$ parameterized by $\boldsymbol{\theta}$.

In this paper, we will take $p_0$ to be 15 MeV, so that we directly regress the reciprocal of the radiation length $\frac{1}{L_{rad}(x, y, z)}$ in units of cm$^{-1}$.

\subsection{Motivation}

Some key features of the muon reconstruction problem are immediately apparent from the problem statement, which help us design our model architecture are listed below:

\begin{itemize}
	\item The model should be permutation-invariant, i.e. the order in which muons are inputted into the model should not change the output.
	\item The model should accept any number of input muons.
	\item The model should be able to make use of paired input and output muon detections i.e. the model needs to know which input muon corresponds to which output.
	\item The model should be able to take advantage of the 3D spatial structure of the target output.
\end{itemize}

This seems to suggest making use of Transformers \cite{mha} with the positional encoding removed. However, the dot product attention module used in Transformers has a large time-complexity of $O(N^2)$ where $N$ is the dosage. This is not acceptable given dosages can go up to 10000 muons or more. The transformer will also be unable to take advantage of the spatial structure of the output. Hence, we will make use of the PoCA algorithm and the U-Net architecture \cite{unet} to create a two-stage algorithm.

\section{Related Work}

\subsection{Deep Learning Methods}

\paragraph{Point Clouds.} From the key features of the muon tomography problem, we see that it is highly similar to point cloud problems, which also take in permutation invariant data (a set of points), but the points themselves also exhibit some 3D structure. Deep learning methods on point cloud data can generally be classified into 2 main categories, neural networks which operates directly on the points, and neural networks which operate on a voxelized representation of the data \cite{bello2020deep}. In our work, we chose the latter approach since the former tends to be slightly worse at capturing spatial structure, and our desired output is also in the form of voxels.

\paragraph{U-Net.} The U-Net was first proposed as an medical image segmentation model \cite{unet}. It consists of a downward branch, where the image is downsampled and an upward branch, where the image is upsampled. Skip connections are also used between the downward and upward branches to allow high resolution features, which may be removed during downsampling, to be retained. U-Nets have also been extended to process 3D data \cite{cciccek20163d, ho2021point}, by replacing the standard 2D convolution operations with 3D convolutions. 

\paragraph{ConvNext.} ConvNext is a recent iteration of the family of convolutional neural networks \cite{convnext}. It was proposed as an improvement of the ResNet \cite{resnet} by incorporating methods found in Vision Transformer architectures, most notably the SWIN transformer \cite{liu2021swin}, and has been shown to obtain competitive performance against Transformer-based methods but with a lower parameter count. In our work, we will use a modified 3D ConvNext block to form our U-Net.

\paragraph{IEEE BigData 2023 Cup.} Concurrent to this work, is the "IEEE BigData 2023 Cup: Object Recognition with Muon Tomography using Cosmic Rays" organised by \citet{competition}. Unlike our work which addresses full three-dimensional reconstruction, they focus on reconstruction of 2-dimensional objects placed in the central plane. The metric used is the mean average IoU. However, this metric equally penalises models for misidentifying materials with a small difference in radiation length and a large difference in radiation length. Hence, in this paper, we adopt the mean squared error and the peak-signal to noise ratio (PSNR) as our primary metrics. We are unfortunately unable to compare our results to the results from this competition as at the time of writing, the competition report and the papers of the winning entries are not accessible online.

\subsection{Traditional Algorithms}

\paragraph{Point of Closest Approach.} The Point of Closest Approach (PoCA) method is a commonly used method for muon scattering tomography, first proposed by \citet{poca}. It assumes that the muon is scattered once by an object at the point of closest approach between the input and output trajectories. However, this fails to take into account the fact that in many cases, muons may scatter multiple times. As a result, there will be predictions that muons scatter at points where there is actually no object as illustrated in Figure \ref{fig:poca}. In addition, some information is lost as some muons may have a point of closest approach outside of the object space.

\begin{figure}[h]
	\centering
	\includegraphics[scale=0.5]{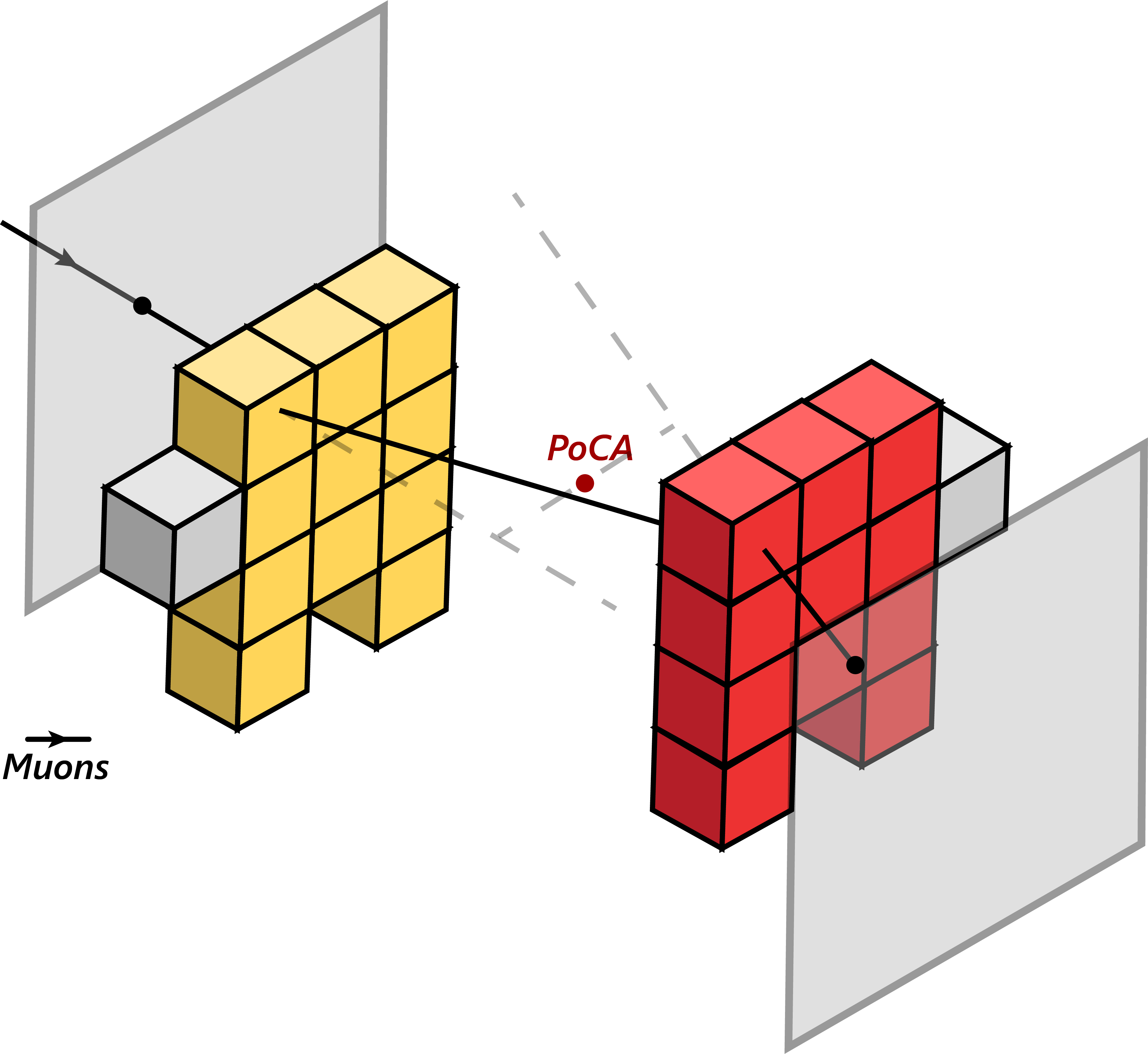}
	\caption{\textbf{A weakness of PoCA.} Since the muon scatters more than once, the computed PoCA is not within any object.}
	\label{fig:poca}
\end{figure}

\paragraph{Maximum Likelihood Expectation Maximization.}

The maximum likelihood expectation maximisation (MLEM) algorithm is a statistical reconstruction method proposed by \citet{mlem}. It makes use of the statistical distribution of muon scattering to frame muon reconstruction as a maximum likelihood problem. An iterative expectation maximisation algorithm is then used to find the scattering densities within the object that is most likely to result in the observed data collected by the muon detectors.

Existing literature \cite{zeng2020principle} have found that MLEM has better qualitative performance than direct allocation to PoCA. However, in our experiments, we observed that MLEM has significantly lower performance (see Figure \ref{fig:2d-crosssection}). We hypothesize that this is due to the lower muon dosages we used in our experiments, which makes it difficult to apply statistical methods to the reconstruction. Due to the lower performance and higher computational cost of MLEM, we will be comparing our model against PoCA for most of our experiments.

\section{Methods}

\begin{figure*}
	\centering
	\includegraphics[scale=0.25]{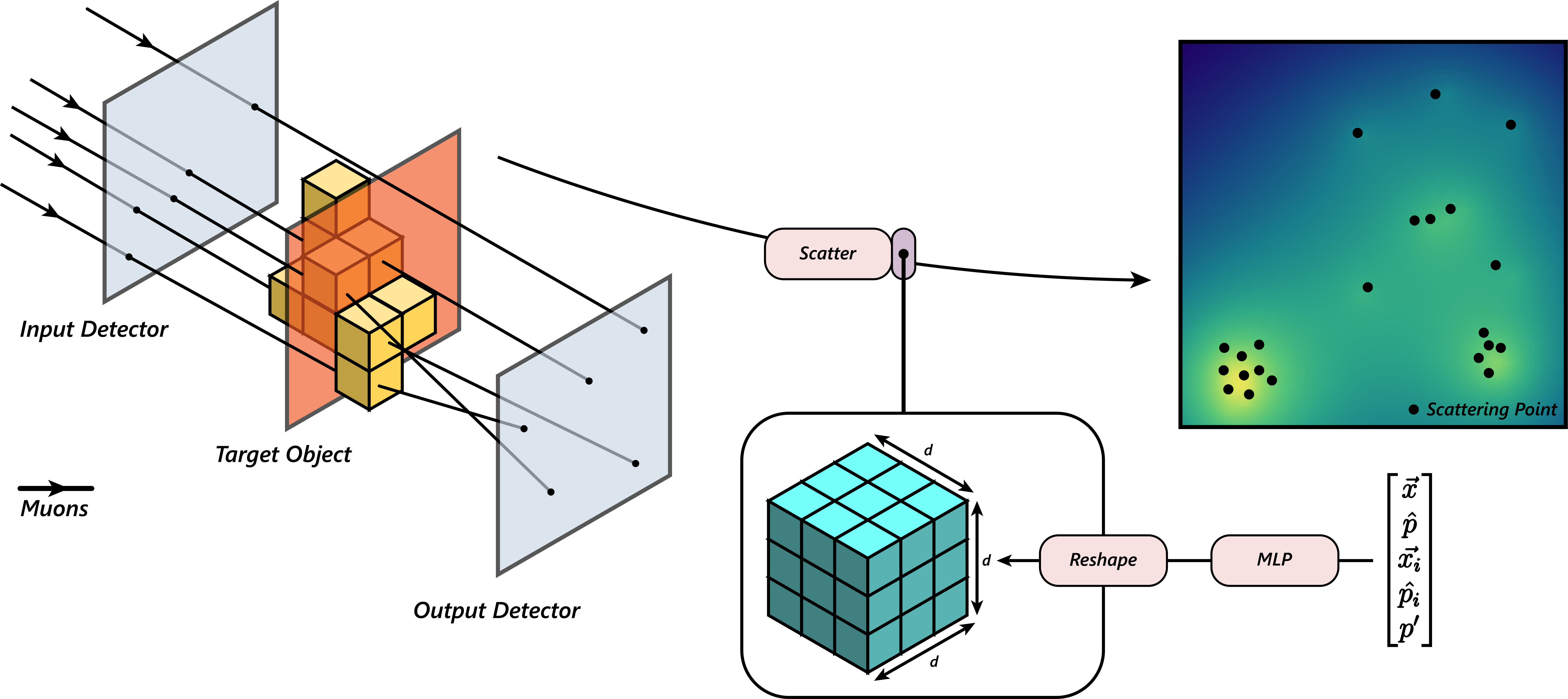}
	\caption{\textbf{First stage of $\boldsymbol{\mu}$-Net.} The muon features (initial position, initial momentum, etc.) are passed through an MLP and placed at PoCA scattering points within a 3D volume. If they overlap, the average is taken.}
	\label{fig:sparse-matrix}
\end{figure*}

\paragraph{Scatter Operation.} Our goal is to convert a set of muon detections into a reconstruction output. One simple way to do this is to first convert the muon detections into some 3D representation, which can then be fed into a U-Net.

To achieve this, we first apply an MLP on the muon's input parameters. The output features are reshaped into a $d\times d\times d\times c$ block. We call $d$ the point size and $c$ is the number of channels. Now, using the PoCA algorithm, we find the muon's point of closest approach and scatter the output features into the voxels nearest to the PoCA point. We choose the PoCA point as the information about the muons that scatter in a given area is the primary information that is needed to help decide what the scattering density of that area is. For muons that do not scatter, we place this block at a random point along their trajectory. This provides improved performance since the vast majority of muons will not scatter. Placing them randomly along their trajectory provides additional information to the model that the scattering density in that region is low. 

After this, if there is overlap in the scattering voxels, the sum is taken and a counter keeps track of how muons scatter in a given voxel. An MLP is used to combine this counter with the other projected features.

This approach can be implemented in a memory-efficient manner using TensorFlow's $\mathtt{tf.scatter\_nd}$ to avoid creating a large number of 3D tensors in parallel. In contrast, other possibilities like converting each muon detection into its own 3D tensor and then summing them will take up alot of memory and is in practise, infeasible.

Nevertheless, this approach does come with some limitations. The most prominent being that the scattering operation is fundamentally non-differentiable. This means that we are unable to parameterise the position of the scattering points using a neural network directly. As such, we resort to using PoCA to get an approximate scattering point for the muons.

\paragraph{Feature Engineering.} Each muon detection is a 14-length vector consisting of the input position, $\mathbf{x_0}$, output position, $\mathbf{x_f}$, input momentum, $\mathbf{\hat{p}_0}$, output momentum, $\mathbf{\hat{p}_f}$, an estimate of the momentum magnitude, $|\mathbf{p}|$ and an estimate of the scattering angle, $|\mathbf{\hat{p}_f}-\mathbf{\hat{p}_i}|$. Although this estimate of the scattering angle can be easily computed from the other features, its inclusion improves performance as the model can easily know which muons should have a larger activation because they scattered more.

\begin{algorithm}[tb]
	\caption{Scattering Operation}
	\label{alg:mu-net-part-1}
	\begin{algorithmic}
		\STATE {\bfseries Input:} muon detections $\{\boldsymbol{\mu}_1,\dots,\boldsymbol{\mu}_n\}$, resolution $R$
		\STATE Initialize $\mathbf{X}\in\mathbb{R}^{r\times r\times r\times (c+1)}$
		\FOR{$i=1$ {\bfseries to} $i=n$}
		\STATE $\mathbf{y}_i\gets\mathsf{MLP}_1(\boldsymbol{\mu}_i),\mathbf{y}_i\in\mathbb{R}^{d\times d\times d\times c}$
		\IF{$\boldsymbol{\mu}_i$ scatters}
		\STATE Place $\mathbf{y}_i$ in X at $\mathsf{PoCA}(\boldsymbol{\mu}_i)$
		\ELSE 
		\STATE Place $\mathbf{y}_i$ at a random point along the muon's trajectory
		\ENDIF
		\STATE Increment the corresponding last channel of X by 1 in the corresponding place
		\ENDFOR
		\STATE $X^\prime\gets\mathsf{MLP}_2(X),X^\prime\in\mathbb{R}^{r\times r\times r\times c}$ 
		\STATE \textbf{return} $X^\prime$
	\end{algorithmic}
\end{algorithm}

\paragraph{U-Net.} Now, we make use of the U-Net \cite{unet} to process the voxelised volume matrix from the first stage. In our U-Net, instead of using the more traditional Residual Blocks \cite{resnet}, we make use of ConvNeXt Blocks \cite{convnext} which result in better performance than Residual Blocks at a lower computational cost. Each layer of the U-Net contains multiple such ConvNeXt blocks. Downsampling is done using Layer Normalisation followed by a convolutional layer of strides = 2. Upsampling is done by first applying a pointwise convolution and layer normalisation before using nearest neighbour upsampling.

\paragraph{Model Sizes.} In our experiments, we tested out 3 models - $\mu$-Net-T, $\mu$-Net-B and $\mu$-Net-L. These variants only differ in the number of blocks $B$ and the number of channels $C$ in each stage. The number of parameters of each of these models is listed as $P$. We do not scale up the model further due to limited computational resources. We will leave further exploration of the scaling of $\mu$-Net to future work.

\begin{itemize}
	\item $\boldsymbol{\mu}$\textbf{-Net-T:} $P=1.1M$\\$B=(1,2,3,4,5),\;C=(8,16,32,64,128)$
	\item $\boldsymbol{\mu}$\textbf{-Net-B:} $P=5.0M$\\$B=(1,2,4,4,6),\;C=(16,32,64,128,256)$
	\item $\boldsymbol{\mu}$\textbf{-Net-L:} $P=14.8M$\\$B=(1,2,4,6,8),\;C=(24,48,96,192,384)$
\end{itemize}

\paragraph{Training Techniques.} We train our model using the AdamW optimizer with a learning rate of $2\times10^{-3}$ and a weight decay of $4\times10^-3$. The model is trained for 15 epochs.

\paragraph{Universal Approximation.} We have shown that $\mu$-Net is an universal function approximator for continuous set functions given the model is large enough and either of the following conditions are met:

\begin{itemize}
	\item The resolution of the reconstructed volume is sufficiently large such that there is no overlap in the reconstructed points \textbf{or}, 
	\item the resolution of the reconstructed volume is finite but the point size is large enough to cover the entire reconstructed volume.
\end{itemize}

Formally, we can define $\chi=\{S: S\subseteq\mathbb{R}^m\;\mathit{and}\;|S| = n\}$ and $f: \chi\rightarrow\mathbb{R}$ is a continuous set function w.r.t the Hausdorff distance $d_H(\cdot,\cdot)$. Our theorem proves that $f$ can be arbitrarily approximated by our model if the resolution is sufficiently high, or if the resolution is fixed but the point size $d$ is the same as the resolution.

\begin{theorem}
	\label{universal_approx}
	Suppose $f: \chi\rightarrow\mathbb{R}^p$ is a continuous set function w.r.t $d_H(\cdot,\cdot)$, such that for all $\epsilon > 0$, there exists some configuration of the model parameters $\theta$ for sufficiently large $p$ \textbf{or} $\phi(\eta(x_i))=J_{p\times d}$ (i.e. the indicator function maps to every point), such that for any $S\in\chi$,
	
	\begin{multline}
		\left|f(S)-\gamma_\theta\left(\left[\sum_{x_i\in S}\{\phi(\eta(x_i))\cdot h_\theta(x_i)\},\right.\right.\right.\\
		\left.\left.\left.\sum_{x_i\in S}\{\phi(\eta(x_i))\cdot J_{d\times c}\}\right]\right)\right| < \epsilon
	\end{multline}
	
	where $\gamma_\theta:\mathbb{R}^{p\times c}\rightarrow\mathbb{R}^p$ is any continuous function, $h_\theta:\mathbb{R}^m\rightarrow\mathbb{R}^{d\times c}$ is any continuous function, $\eta:\mathbb{R}^m\rightarrow\mathbb{R}$, $\phi:\mathbb{R}^m\rightarrow\mathbb{R}^{p\times d}$ and $J_{d\times c}$ is the ones matrix of shape $(d, c)$. $\eta$ represents the PoCA function that generates a scattering point from the muon detection. $\phi$ is an indicator function for a set of intervals of length $d$ derived from its input. The indicator function for each of these intervals is placed along one row in the last dimension. $\gamma_\theta$ and $h_\theta$ can be taken to be any continuous function due to the universal approximation theorem for CNNs and MLPs. $[\textbf{A}, \textbf{B}]$ represents the concatenation of 2 matrices along their last axes.
	
\end{theorem}

A brief proof is provided in the Appendix. We also note this theorem generalises to all dimensions easily by replacing $p$ with $r\times r$, $r\times r\times r$, etc.

The limitations of the theorem to arbitrarily large resolutions or point sizes comes from our lack of assumptions about the PoCA and scatter operations (i.e. $\phi(\eta(x_i))$. The nature of these functions depends on the interactions between the muons and the object and is very difficult to analyse due to the complex interactions that muons exhibit with matter.

However, despite these limitations, we hypothesise that our model performs well with a fixed resolution and point size since most of the information about the muon scattering is contained near the the scattering point.

\renewcommand\qedsymbol{$\blacksquare$}

\section{Experiments}

\subsection{Experimental Setup}

Our data is generated using CERN's Geant4 simulation software \cite{geant4}. To generate 3D objects to be analysed by our system, we make use of fractal noise to generate objects of various shapes. The materials of the object is randomly chosen from a set list of materials of different radiation lengths, taken from the IEEE BigData Competition on Muon Tomography \cite{competition}.

The geometry of the system can be found in Figure \ref{fig:geometry}. The target object is contained within a cube of side length 1 m. The input and output detectors are squares of side length 2 m. They are separated from the object by a distance of 0.5 m.

\begin{figure}[h]
	\centering
	\includegraphics[scale=0.3]{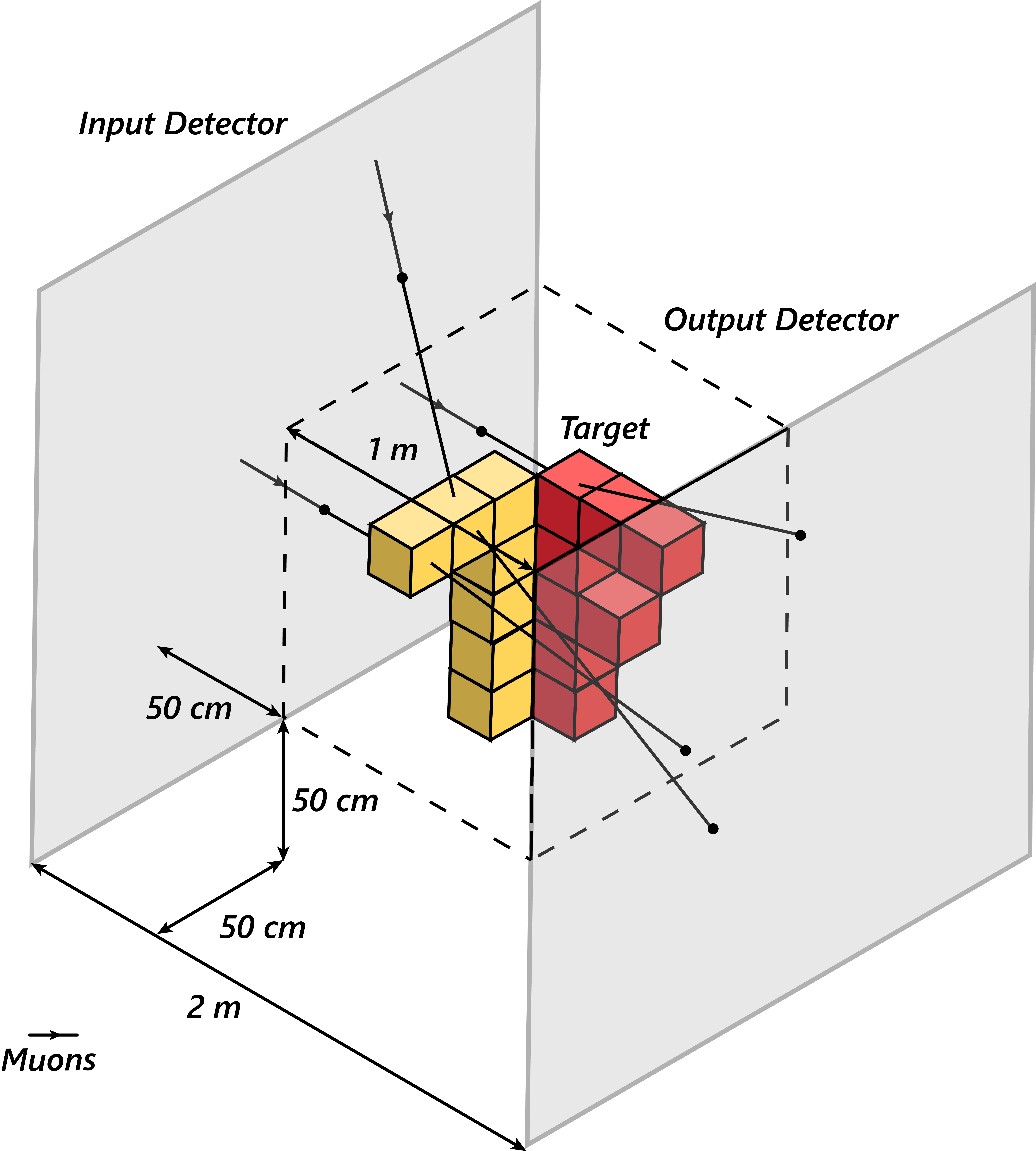}
	\caption{\textbf{Simulation setup.} A scale diagram (except the voxels) illustrating the simulation setup within Geant4.}
	\label{fig:geometry}
\end{figure}

For the muon beam, we use a beam with a $\cos^2$ angular distribution and a power law distribution, in accordance with characterised values of the cosmic muon flux \cite{muon_flux_2}. Muons that are calculated not to hit the detector are killed at the start of the simulation to ensure the simulation runs at a reasonable speed for data generation. 

We use 20000 samples for the training set, 1600 samples for the validation set and 1600 samples for the test set. Our model is implemented using TensorFlow and trained using 2 T4 GPUs.

\subsection{Ablations}

\paragraph{Point Size.} First, we vary the point size at various dosages to see its impact on the model's performance. Interestingly, we find that a smaller point size of 1 results in the best performance.

\begin{table}
	\caption{\textbf{Smaller is better.} The results of the model at different dosages for various point sizes. The inference times are evaluated on 2 T4 GPUs with a batch size of 8. The best results are bolded. The smaller point size of 1 outperforms the larger point size of 3 on almost all dosages.}
	\label{tab:point-size}
	\centering
	\vskip 0.15in
	\begin{tabular}{p{0.7cm}lllll}
		\toprule
		\textbf{Point Size} & \textbf{Dosage} & \textbf{Time}$\downarrow$ &  \textbf{MSE}$\downarrow$ & \textbf{MAE}$\downarrow$ & \textbf{PSNR}$\uparrow$ \\
		\midrule
		1 & 1024 & \textbf{126 ms} & \textbf{0.2276} & \textbf{0.2204} & \textbf{17.1426} \\
		3 & 1024 & 134 ms & 0.2289 & 0.2366 & 17.0971 \\
		\midrule
		1 & 2048 & \textbf{135 ms} & 0.1965 & 0.1989 & 17.7786 \\
		3 & 2048 & 143 ms & \textbf{0.1947} & \textbf{0.1949} & \textbf{17.8280} \\
		\midrule
		1 & 4096 & \textbf{141 ms} & \textbf{0.1653} & \textbf{0.1725} & \textbf{18.5347} \\
		3 & 4096 & 178 ms & 0.1685 & 0.1741 & 18.4482 \\
		\midrule
		1 & 8192 & \textbf{169 ms} & \textbf{0.1388} & \textbf{0.1438} & \textbf{19.2979} \\
		3 & 8192 & 219 ms & 0.1403 & 0.1465 & 19.2434 \\
		\midrule
		1 & 16384 & \textbf{246 ms} & \textbf{0.1169} & \textbf{0.1207} & \textbf{20.0433} \\
		3 & 16384 & 318 ms & 0.1205 & 0.1340 & 19.9047 \\
		\midrule
		1 & 32768 & \textbf{347 ms} & \textbf{0.0993} & \textbf{0.1156} & \textbf{20.7906} \\
		3 & 32768 & 574 ms & 0.1048 & 0.1224 & 20.4946 \\
		\bottomrule
	\end{tabular}
\end{table}

\paragraph{Estimate of Scattering Angle.} We then test out the impact of not providing an estimate of the scattering angle. For a dosage of 1024 muons using $\mu$-Net-T, this results in a PSNR of 16.9142, in contrast with a PSNR of 17.1426 when it is provided. This demonstrates that including the scattering angle, although it can be computed from the other features in the muon detection, helps to improve the model's accuracy.

\paragraph{Random Placement of Muons.} We also test out the impact of placing the non-scattered muons at the centre of the their trajectory, rather than at a random point along their trajectory. For a dosage of 1024 muons using $\mu$-Net-T, this results in a PSNR of 16.8650, in contrast with a PSNR of 17.1426 when the muons are placed randomly along their trajectory. This similarly demonstrates the effectiveness of this method in spreading out the information about unscattered muons across the 3D voxels, enabling a more accurate reconstruction.

\subsection{Model Scaling}

\begin{figure}[h]
	\centering
	\includegraphics[scale=0.5]{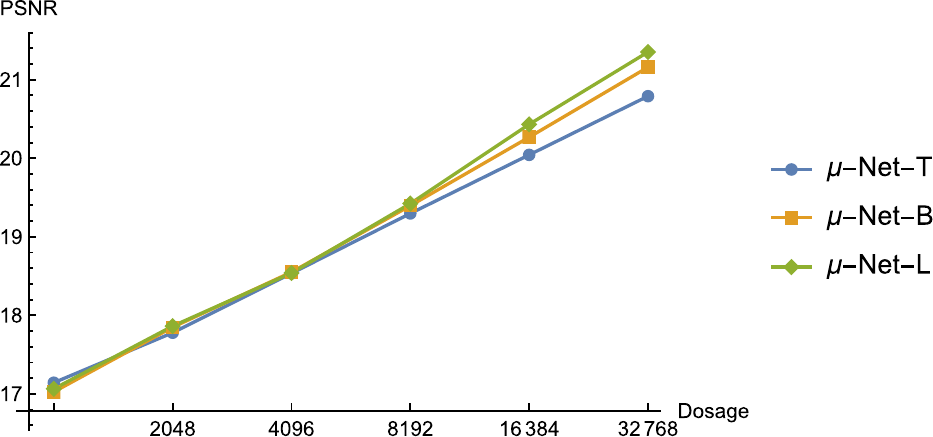}
	\caption{\textbf{Model scaling has little impact.} The PSNR of the different model sizes plotted against dosage levels. We see that for low dosages, the model size has little impact. However, at larger dosages, the impact of model size is greater.}
	\label{fig:model-scaling}
\end{figure}

Now, we explore the impact of the scaling of the model on the performance. The results are plotted in Figure \ref{fig:model-scaling}. We notice that the improvements in performance are actually quite small, especially for small dosages. We hypothesise that this is because we are near the "limit" of how good the reconstruction can be given the U-Net's input of the PoCA points. It is likely that attaining a better estimate of the scattering points would lead to better performance.

Furthermore, we notice that at larger dosages, the improvements in performance increase, likely due to the problem being more complex, hence, providing more room for improvement with a larger model.

\subsection{Comparison with Traditional Algorithms}

\begin{figure}[h]
	\centering
	\includegraphics[scale=0.5]{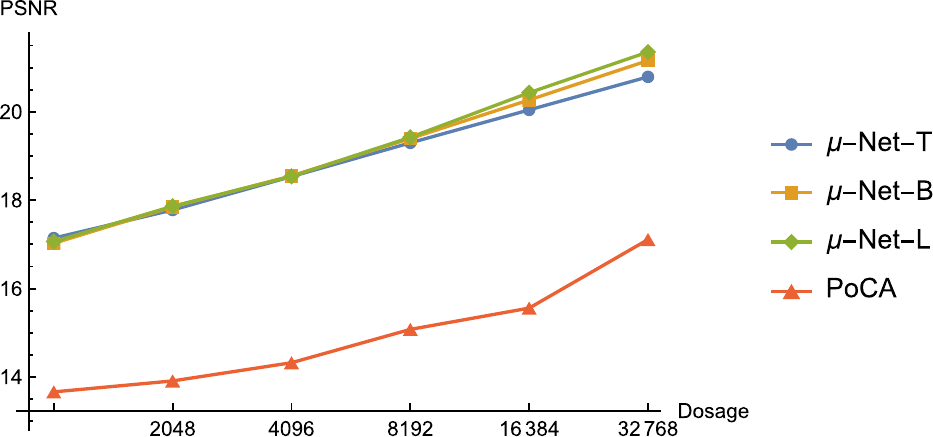}
	\caption{\textbf{$\boldsymbol{\mu}$-Net outperforms PoCA.} The PSNR of $\mu$-Net and PoCA plotted against dosage levels. $\mu$-Net consistently outperforms PoCA over all dosages and all model sizes.}
	\label{fig:poca-comparison}
\end{figure}

\paragraph{Dosage.} We vary the dosage of muons from 1024 to 32768. The results are shown in Figure \ref{fig:model-scaling}. We see that as the dosage increases, the performance of the models increase as well. It is also clear that our model, $\mu$-Net, significantly outperforms the the traditional PoCA algorithm in Figure \ref{fig:poca-comparison}. However, we notice that at the gradient of the graph for PoCA appearing to be increasing. Future study is needed to ascertain if this is just a statistical anomaly or if the PoCA indeed curves upwards.

\begin{figure}[h]
	\includegraphics[scale=0.5]{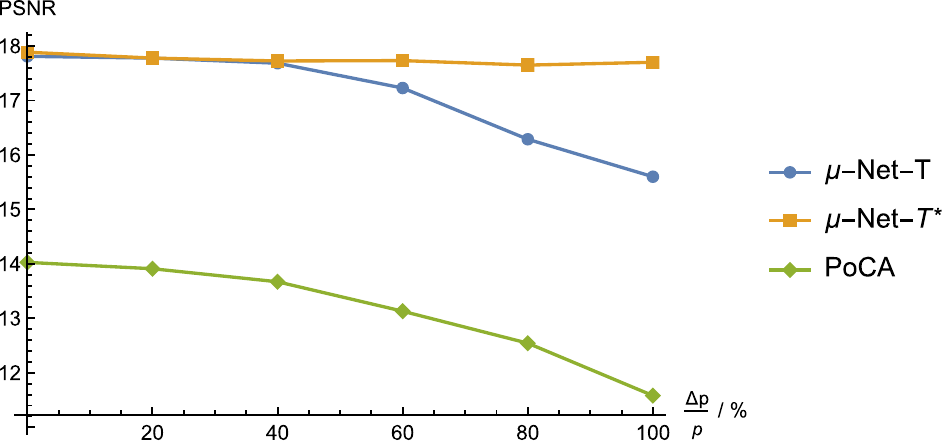}
	\caption{\textbf{$\boldsymbol{\mu}$-Net is robust to momentum error.} The PSNR of various methods plotted against percentage error in the momentum estimate. $\mu$-Net$^*$ represents the model's performance when it is finetuned on the new data. Our model is consistently shown to be robust to perturbations and to significantly outperform the PoCA algorithm.}
	\label{fig:momentum-error}
\end{figure}

\paragraph{Momentum Estimate.} We also look at how varying levels of error in the momentum will affect predictions in Figure \ref{fig:momentum-error}. We again find that our model significantly outperforms the traditional PoCA algorithm. After fine-tuning, we also see that the model's performance stays relatively constant as the error in the momentum increases, indicating our model is robust.

\begin{figure}[h]
	\includegraphics[scale=0.5]{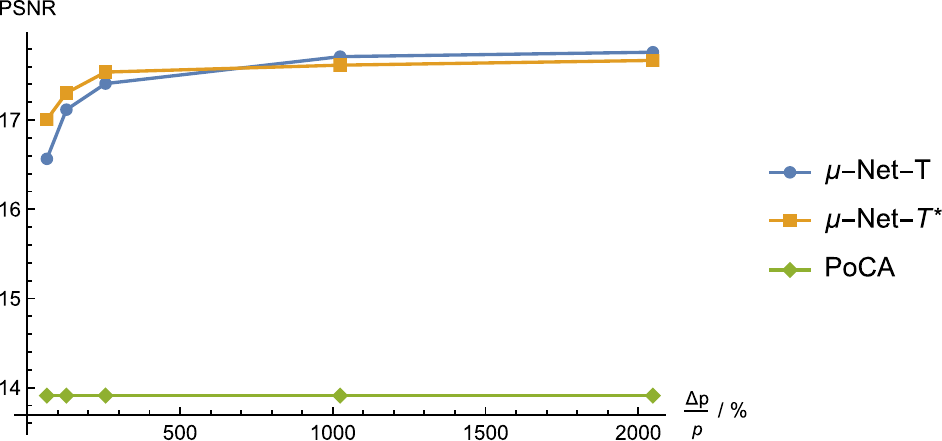}
	\caption{\textbf{$\boldsymbol{\mu}$-Net is robust to detector resolution.} The PSNR of various methods plotted against detector resolution. $\mu$-Net$^*$ represents the model's performance when it is finetuned on the new data. Our model is consistently shown to be robust to perturbations and to significantly outperform the PoCA algorithm.}
	\label{fig:detector-resolution}
\end{figure}

\paragraph{Detector Resolution.} Finally, we look at how the model's performance changes with the detector resolution in Figure \ref{fig:detector-resolution}. Again, we find that our model significantly outperforms the PoCA algorithm. In addition, we find that our model performs well at a variety of resolutions, showing that it is very robust. We also notice that the performance of PoCA stays constant across all resolutions and this is because the reconstruction volume has a resolution of $64\times64\times64$.

\begin{figure*}
	\centering
	\includegraphics[scale=1.5]{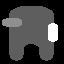}
	\hspace{0.2cm}
	\includegraphics[scale=1.5]{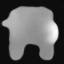}
	\hspace{0.2cm}
	\includegraphics[scale=1.5]{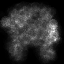}
	\hspace{0.2cm}
	\includegraphics[scale=2.1]{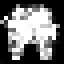}
	\caption{\textbf{2D Reconstruction Results.} This is a cross-section of an object made of 3 different materials. The ability of the model to reconstruct the approximate shapes and differentiate materials is shown clearly. However, there is still a significant amount of blurring. (left) ground truth, (middle-left) $\mu$-Net, (middle-right) PoCA (right) MLEM. MLEM was performed at a much lower resolution because it requires muon tracks to intersect with every voxel.}
	\label{fig:amongus}
\end{figure*}

\begin{figure*}
	\centering
	\includegraphics[scale=2]{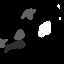}
	\hspace{0.2cm}
	\includegraphics[scale=2]{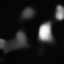}
	\hspace{0.2cm}
	\includegraphics[scale=2]{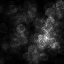}
	\includegraphics[trim={6.1cm 2cm 7cm 4cm},scale=0.35,clip]{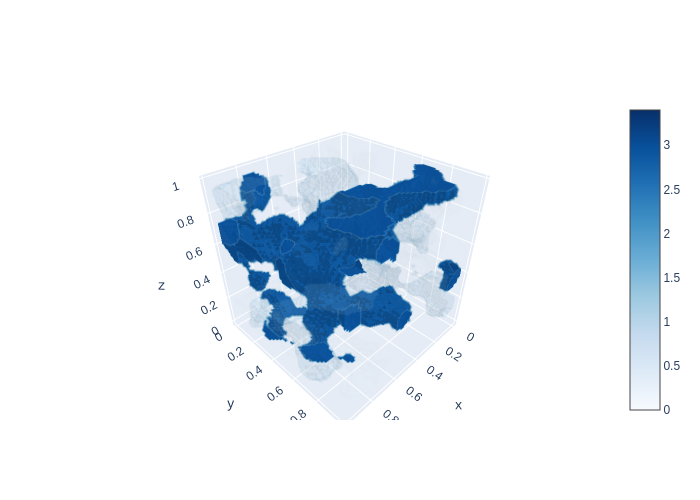}
	\includegraphics[trim={6.1cm 2cm 7cm 4cm},scale=0.35,clip]{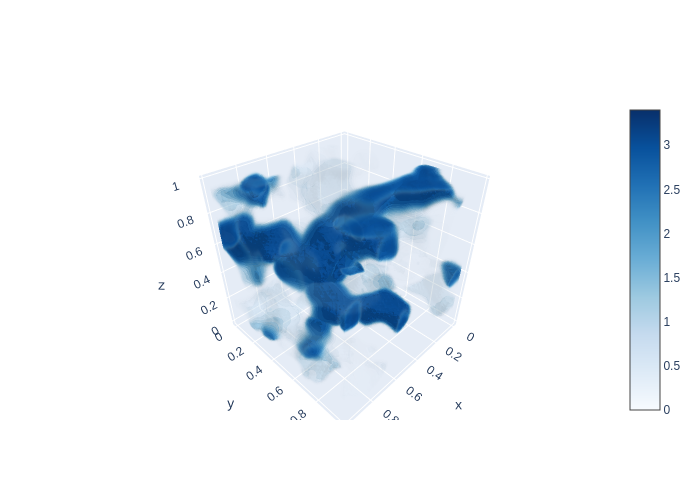}
	\includegraphics[trim={6.1cm 2cm 7cm 4cm},scale=0.35,clip]{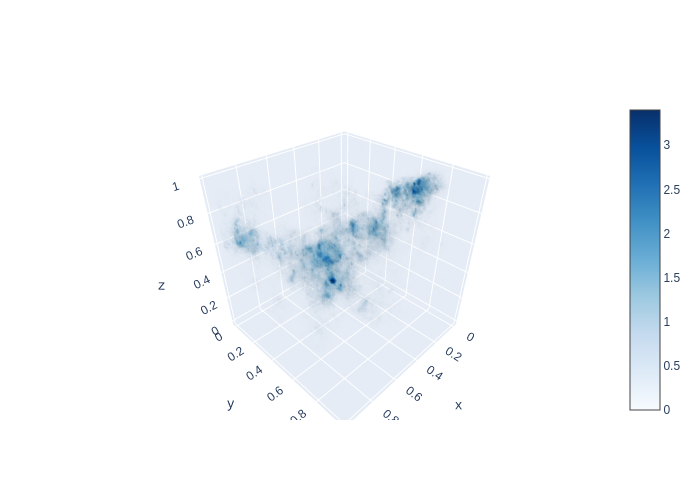}
	\caption{\textbf{3D Reconstruction Results.} This is a 3D plot and a cross-section of one of the imaging targets from the testing set. The ability of the model to reconstruct the approximate shapes and differentiate materials is shown clearly. However, there is still a significant amount of blurring. (left) ground truth, (middle) $\mu$-Net, (right) PoCA}
	\label{fig:3d-reconstruction-main}
\end{figure*}

\paragraph{Visual Comparison.} Now, we visually compare the reconstructions of PoCA and MLEM. These reconstructions are shown in Figure \ref{fig:amongus} and \ref{fig:3d-reconstruction-main}. More reconstructions from different dosages can be found in Appendix \ref{app:images}. A dosage of 32768 muons is used. The MLEM reconstruction is significantly worse since it had to be done using a lower resolution, because the algorithm requires the muon tracks to pass through every voxel, which is not possible at higher resolutions.

We see that $\mu$-Net provides superior reconstruction quality compared to PoCA. Furthermore, $\mu$-Net is also able to distinguish different materials by their radiation length and accurately reconstruct the approximate shape of objects. Nevertheless, we still observe significant levels of blurring. Further analysis of the artifacts found in the reconstruction can be found in Appendix \ref{app:artifacts}.

\section{Discussion}

\subsection{Applications}

\paragraph{Nuclear Non-proliferation.} Muon tomography has been applied in detecting the presence of high-Z materials such as radioactive materials since these materials result in a large amount of scattering. It can also be applied for general screening of cargo for high-Z materials \cite{border-security}. Our model can be used to help improve the quality of these reconstructions, potentially reducing the dosage required to screen the cargo, enabling lower screening times.

\paragraph{Archaeology.} In addition, muon tomography also has numerous applications in archaeology. In particular, the ability of muons to easily penetrate thick layers of material such as rock, enables their use to image the interiors of large structures. For instance, researchers have used muon tomography to discover secret chambers within Khufu's Pyramid \cite{giza, giza2}, image underground cavities in Mount Echia \cite{mountain, mountain2} and discover a secret ancient Greek burial chamber in the centre of Naples \cite{naples}. In cases where the input detections are available, our model can be used to significantly improve the accuracy of the reconstructions.

\subsection{Future Work}

\paragraph{No Input Muon Information.} One limitation of the model is that it depends on the presence of information about the input muons because of its use of the PoCA algorithm. However, in some cases, such as in some archaeological applications, the direction of the input muons is unavailable.

\paragraph{Trajectory Prediction.} As shown in Figure \ref{fig:poca}, when muons scatter more than once, their PoCA will be outside of the boundaries of any object. This is a significant limitation of PoCA and sometimes results in false hotspots of scattering density in the final prediction. One solution to this could be to attempt to predict the muon's full trajectory using a method like \cite{mpt} given an initial guess of the scattering density, which can be obtained from our current model. Using this information, more accurate coordinates of scattering points can be obtained, allowing a more accurate scattering density to be obtained. This process can then be iterated until convergence.

\paragraph{Out-of-Distribution Generalisation.} It is not clear to what extent the model's performance is independent of factors such as the angular and spacial distribution of cosmic ray muons. In addition, in real life, the spacial distribution of the scattering densities will not be that of fractal noise, like what the model was trained on. Before such a model can be deployed in the real-world, it would be important to check its ability to generalise.

\section{Conclusion}

In conclusion, we have constructed a state-of-the-art model for muon scattering tomography which outperforms traditional methods such as PoCA and MLEM. Furthermore, we find that our model is robust to various corruptions, with its performance barely changing when they are applied. We hope that our research will spark further investigation into the usage of deep learning in this field. Improvements in imaging techniques for muon scattering tomography will have wide-ranging applications from ensuring nuclear non-proliferation to the discovery of secret chambers in ancient structures.

\section*{Acknowledgments}

The authors would like to thank their friend Kannan Vishal for helping with the writing of an initial simulation in Geant4, their friends Prannaya Gupta, Kabir Jain, Mahir Hitesh for providing computing power for this project and their teachers Mr Silas Yeem Kai Ean and Mr Ng Chee Loong for providing useful advice.

\bibliography{paper}
\bibliographystyle{icml2024}

\newpage
\appendix
\onecolumn
\section{Proofs}

To prove \autoref{universal_approx}, we shall split the theorem into the arbitrarily large resolution case and the arbitrary large point size case.

\subsection{Arbitrarily Large Resolution}

\begin{theorem}
	Suppose $f: \chi\rightarrow\mathbb{R}^p$ is a continuous set function w.r.t $d_H(\cdot,\cdot)$, such that for all $\epsilon > 0$, there exists some configuration of the model parameters $\theta$ for sufficiently large $p$ \textbf{or} $\phi(\eta(x_i))=J_{p\times d}$ (i.e. the indicator function maps to every point), such that for any $S\in\chi$,
	
	\begin{equation*}
		\left|f(S)-\gamma_\theta\left(\left[\sum_{x_i\in S}\{\phi(\eta(x_i))\cdot h_\theta(x_i)\},\sum_{x_i\in S}\{\phi(\eta(x_i))\cdot J_{d\times c}\}\right]\right)\right| < \epsilon
	\end{equation*}
	
	where $\gamma_\theta:\mathbb{R}^{p\times c}\rightarrow\mathbb{R}^p$ is any continuous function, $h_\theta:\mathbb{R}^m\rightarrow\mathbb{R}^{d\times c}$ is any continuous function, $\eta:\mathbb{R}^m\rightarrow\mathbb{R}$, $\phi:\mathbb{R}^m\rightarrow\mathbb{R}^{p\times d}$ and $J_{d\times c}$ is the ones matrix of shape $(d, c)$. $\eta$ represents the PoCA function that generates a scattering point from the muon detection. $\phi$ is an indicator function for a set of intervals of length $d$ derived from its input. The indicator function for each of these intervals is placed along one row in the last dimension. $\gamma_\theta$ and $h_\theta$ can be taken to be any continuous function due to the universal approximation theorem for CNNs and MLPs. $[\textbf{A}, \textbf{B}]$ represents the concatenation of 2 matrices along their last axes.
\end{theorem}

\renewcommand\qedsymbol{$\blacksquare$}
\begin{proof}
	The idea is that for a sufficiently large $p$, the indicator functions of $\phi$ will not overlap, allowing $S$ to be recovered exactly using an inverse function $T$.\\
	
	Let $h_\theta$ simply be the identity function. \footnote{\label{fn}Since $h_\theta$ is the identity function, $c=m$} Now, consider a function $\mathcal{T}: \mathbb{R}^{p\times m}\rightarrow\chi$, $\mathcal{T}(X)=\{x:|x|>0,x\in\mathbb{R}^m,\textit{x is an entry in the last dimension of X}\}$. Now, we define $\gamma_\theta$ as $f\circ\mathcal{T}$. Clearly,
	
	\begin{equation*}
		\left|f(S)-f\left(\mathcal{T}\left(\left[\sum_{x_i\in S}\{\phi(\eta(x_i))\cdot h_\theta(x_i)\},\sum_{x_i\in S}\{\phi(\eta(x_i))\cdot J_{d\times c}\}\right]\right)\right)\right|=\left|f(S)-f(S)\right|=0 < \epsilon
	\end{equation*}
\end{proof}

\subsection{Arbitrarily Large Point Size}

\begin{theorem}
	Suppose $f: \chi\rightarrow\mathbb{R}^p$ is a continuous set function w.r.t $d_H(\cdot,\cdot)$, such that for all $\epsilon > 0$, there exists some configuration of the model parameters $\theta$ for sufficiently large $p$ \textbf{or} $\phi(\eta(x_i))=J_{p\times d}$ (i.e. the indicator function maps to every point), such that for any $S\in\chi$,
	
	\begin{equation*}
		\left|f(S)-\gamma_\theta\left(\left[\sum_{x_i\in S}\{\phi(\eta(x_i))\cdot h_\theta(x_i)\},\sum_{x_i\in S}\{\phi(\eta(x_i))\cdot J_{d\times c}\}\right]\right)\right| < \epsilon
	\end{equation*}
	
	where $\gamma_\theta:\mathbb{R}^{p\times c}\rightarrow\mathbb{R}^p$ is any continuous function, $h_\theta:\mathbb{R}^m\rightarrow\mathbb{R}^{d\times c}$ is any continuous function, $\eta:\mathbb{R}^m\rightarrow\mathbb{R}$, $\phi:\mathbb{R}^m\rightarrow\mathbb{R}^{p\times d}$ and $J_{d\times c}$ is the ones matrix of shape $(d, c)$. $\eta$ represents the PoCA function that generates a scattering point from the muon detection. $\phi$ is an indicator function for a set of intervals of length $d$ derived from its input. The indicator function for each of these intervals is placed along one row in the last dimension. $\gamma_\theta$ and $h_\theta$ can be taken to be any continuous function due to the universal approximation theorem for CNNs and MLPs. $[\textbf{A}, \textbf{B}]$ represents the concatenation of 2 matrices along their last axis.
\end{theorem}

\begin{proof}
	In the case of $\phi(\eta(x_i))=J_{p\times d}$, the theorem reduces to
	
	\begin{equation*}
		\left|f(S)-\gamma_\theta\left(\left[\sum_{x_i\in S}\{ h_\theta(x_i)\}, J_{p\times1}\right]\right)\right| < \epsilon
	\end{equation*}
	
	which can be equivalently expressed as

	\begin{equation*}
		\left|f(S)-\gamma_\theta\left(\sum_{x_i\in S}\{ h_\theta(x_i)\}\right)\right| < \epsilon
	\end{equation*}
	
	given that $\gamma_\theta$ is a universal function approximator.
	
	By invoking Theorem 7 in \cite{NIPS2017_f22e4747}, any continuous permutation invariant set function can be decomposed into the form 
	\begin{equation*}
		f(S)=\rho(\sum_{x_i\in S}\{\psi(x_i)\})
	\end{equation*}
	
	where $\rho$ and $\psi$ are continuous functions. By the universal approximation theorem for CNNs and MLPs, there exists some $\theta$ such that
	
	\begin{equation*}
		\rho(\sum_{x_i\in S}\psi(x_i))=\gamma_\theta\left(\sum_{x_i\in S}\{ h_\theta(x_i)\}\right)
	\end{equation*}
	
	Therefore,
	
	\begin{equation*}
		\left|f(S)-\gamma_\theta\left(\sum_{x_i\in S}\{ h_\theta(x_i)\}\right)\right| = \left| f(S) - f(S) \right| < \epsilon
	\end{equation*}
	
\end{proof}

\newpage

\section{Reconstruction Results}
\label{app:images}

\begin{figure}
	\centering
	\begin{tabular}{ccccccccccc}
		$\sf{Ground\;Truth}$ & $\sf{1024}$ & $\sf{2048}$ & $\sf{4096}$ & $\sf{8192}$ & $\sf{16384}$ & $\sf{32768}$\\
		\vspace{2mm}
		\includegraphics[scale=0.9]{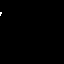} & \includegraphics[scale=0.9]{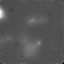} & \includegraphics[scale=0.9]{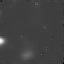} & \includegraphics[scale=0.9]{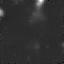} & \includegraphics[scale=0.9]{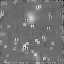} & \includegraphics[scale=0.9]{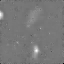} & \includegraphics[scale=0.9]{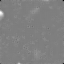} &\\
		\vspace{2mm}
		\includegraphics[scale=0.9]{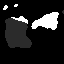} & \includegraphics[scale=0.9]{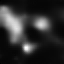} & \includegraphics[scale=0.9]{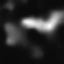} & \includegraphics[scale=0.9]{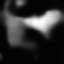} & \includegraphics[scale=0.9]{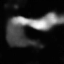} & \includegraphics[scale=0.9]{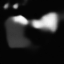} & \includegraphics[scale=0.9]{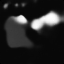} &\\
		\vspace{2mm}
		\includegraphics[scale=0.9]{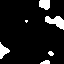} & \includegraphics[scale=0.9]{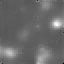} & \includegraphics[scale=0.9]{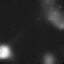} & \includegraphics[scale=0.9]{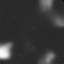} & \includegraphics[scale=0.9]{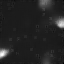} & \includegraphics[scale=0.9]{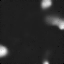} & \includegraphics[scale=0.9]{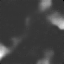} &\\
		\vspace{2mm}
		\includegraphics[scale=0.9]{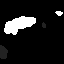} & \includegraphics[scale=0.9]{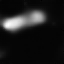} & \includegraphics[scale=0.9]{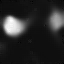} & \includegraphics[scale=0.9]{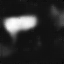} & \includegraphics[scale=0.9]{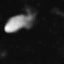} & \includegraphics[scale=0.9]{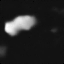} & \includegraphics[scale=0.9]{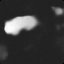} &\\
		\vspace{2mm}
		\includegraphics[scale=0.9]{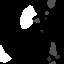} & \includegraphics[scale=0.9]{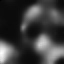} & \includegraphics[scale=0.9]{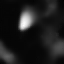} & \includegraphics[scale=0.9]{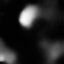} & \includegraphics[scale=0.9]{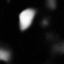} & \includegraphics[scale=0.9]{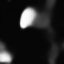} & \includegraphics[scale=0.9]{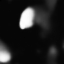} &\\
		\vspace{2mm}
		\includegraphics[scale=0.9]{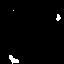} & \includegraphics[scale=0.9]{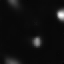} & \includegraphics[scale=0.9]{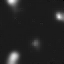} & \includegraphics[scale=0.9]{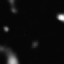} & \includegraphics[scale=0.9]{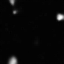} & \includegraphics[scale=0.9]{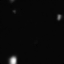} & \includegraphics[scale=0.9]{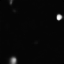} &\\
		\vspace{2mm}
		\includegraphics[scale=0.9]{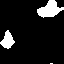} & \includegraphics[scale=0.9]{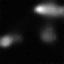} & \includegraphics[scale=0.9]{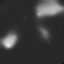} & \includegraphics[scale=0.9]{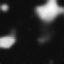} & \includegraphics[scale=0.9]{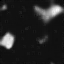} & \includegraphics[scale=0.9]{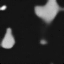} & \includegraphics[scale=0.9]{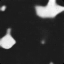} &\\
		\vspace{2mm}
		\includegraphics[scale=0.9]{images/large/1024/crosssection_7.png} & \includegraphics[scale=0.9]{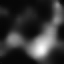} & \includegraphics[scale=0.9]{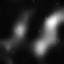} & \includegraphics[scale=0.9]{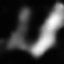} & \includegraphics[scale=0.9]{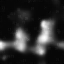} & \includegraphics[scale=0.9]{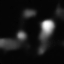} & \includegraphics[scale=0.9]{images/large/32768/crosssection_7_pred.png} &\\
	\end{tabular}
	\caption{\textbf{2D Cross-sections.} 2D cross-sections of the 3D reconstructions produced by $\mu$-Net-L at various dosages. The improvement in reconstruction quality as the dosage increases can be seen clearly. We also see that some cross-sections appear to have worse cross-sections. This is because the materials being reconstructed have a high radiation length, so the muons do not scatter very much.}
	\label{fig:2d-crosssection}
\end{figure}

\begin{figure}
	\centering
	\begin{tabular}{ccccccccccc}
		$\sf{Ground\;Truth}$ & $\sf{1024}$ & $\sf{2048}$ & $\sf{4096}$ \\
		\includegraphics[trim={6.1cm 2cm 7cm 4.5cm},scale=0.35,clip]{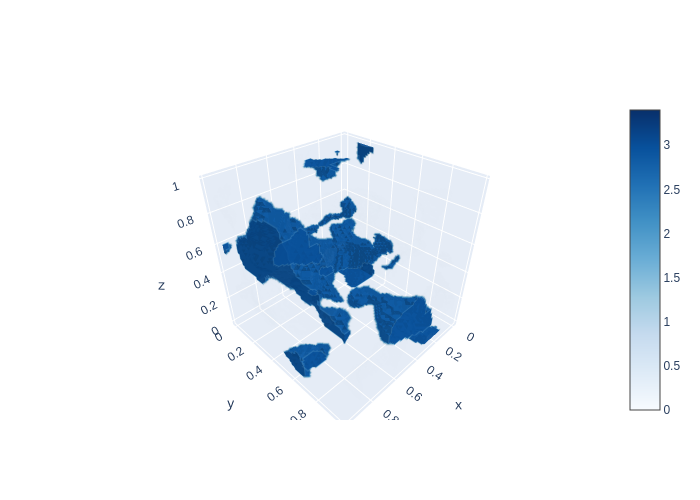} & \includegraphics[trim={6.1cm 2cm 7cm 4.5cm},scale=0.35,clip]{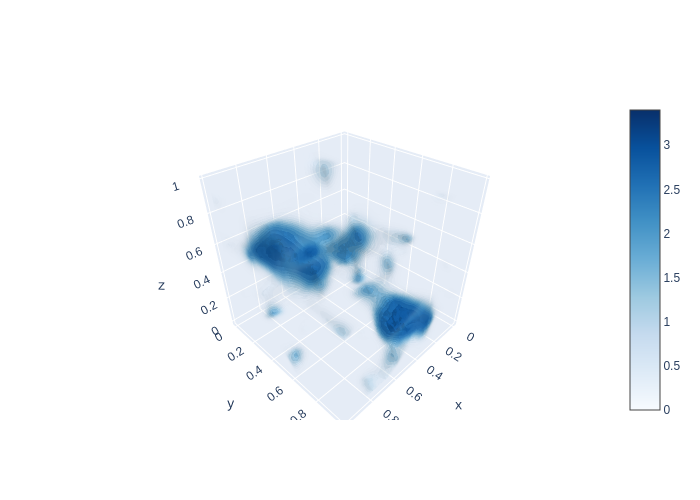} & \includegraphics[trim={6.1cm 2cm 7cm 4.5cm},scale=0.35,clip]{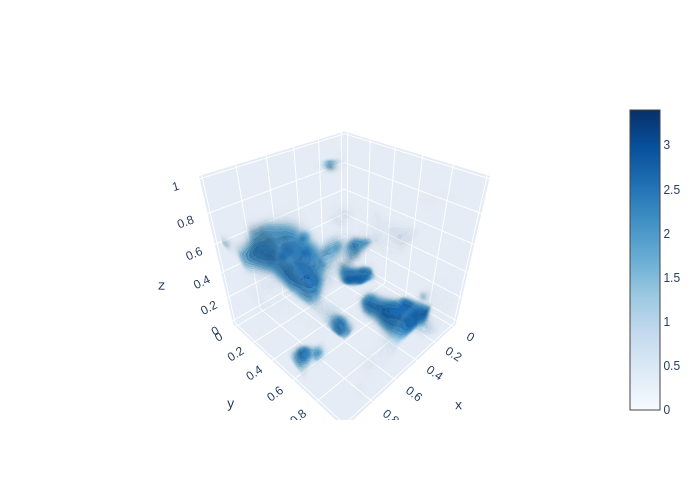} & \includegraphics[trim={6.1cm 2cm 7cm 4.5cm},scale=0.35,clip]{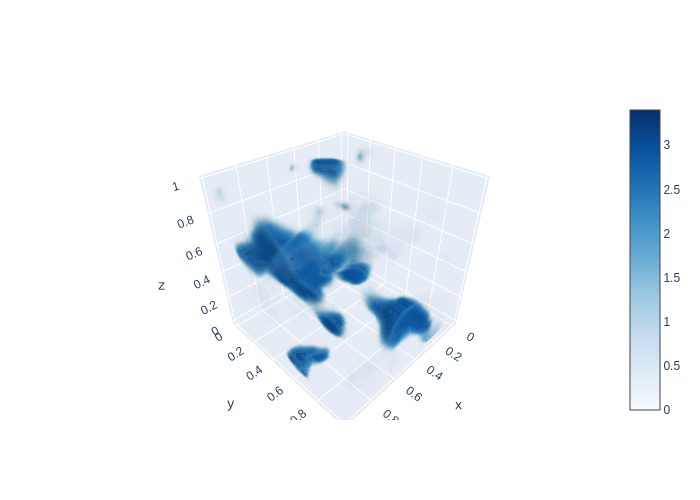} &\\
		& $\sf{8192}$ & $\sf{16384}$ & $\sf{32768}$\\
		& \includegraphics[trim={6.1cm 2cm 7cm 4.5cm},scale=0.35,clip]{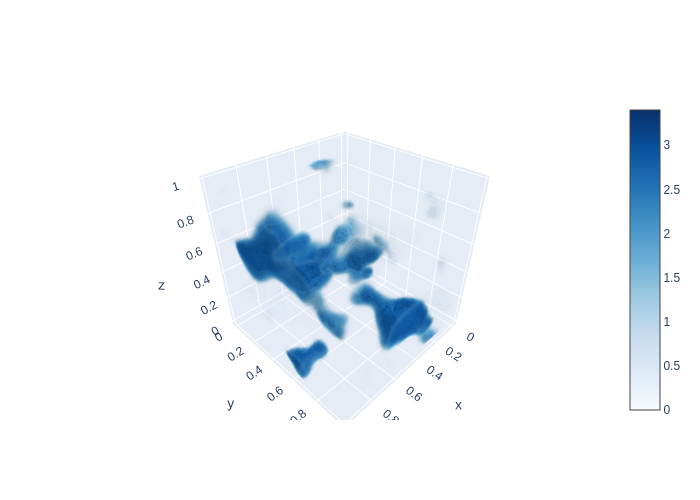} & \includegraphics[trim={6.1cm 2cm 7cm 4.5cm},scale=0.35,clip]{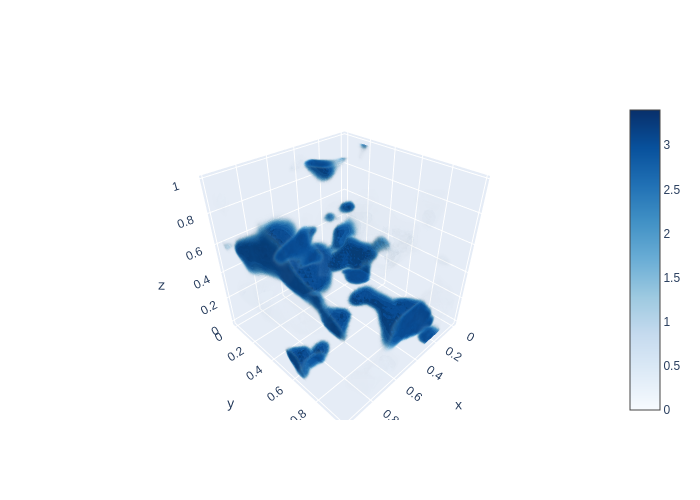} & \includegraphics[trim={6.1cm 2cm 7cm 4.5cm},scale=0.35,clip]{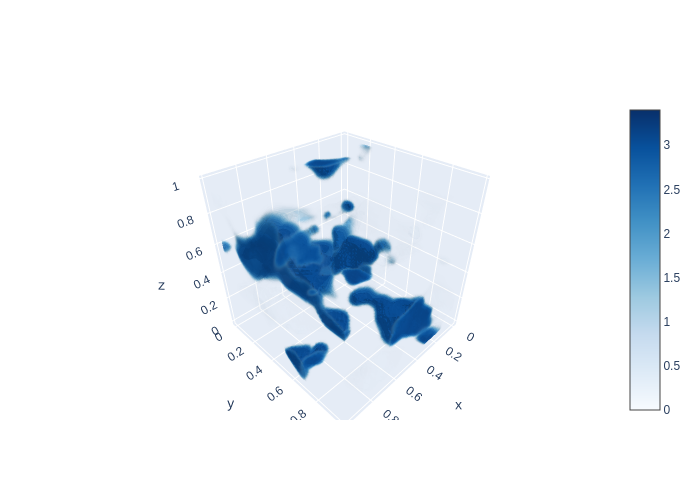} & \\
		\midrule\\
		$\sf{Ground\;Truth}$ & $\sf{1024}$ & $\sf{2048}$ & $\sf{4096}$\\
		\includegraphics[trim={6.1cm 2cm 7cm 4.5cm},scale=0.35,clip]{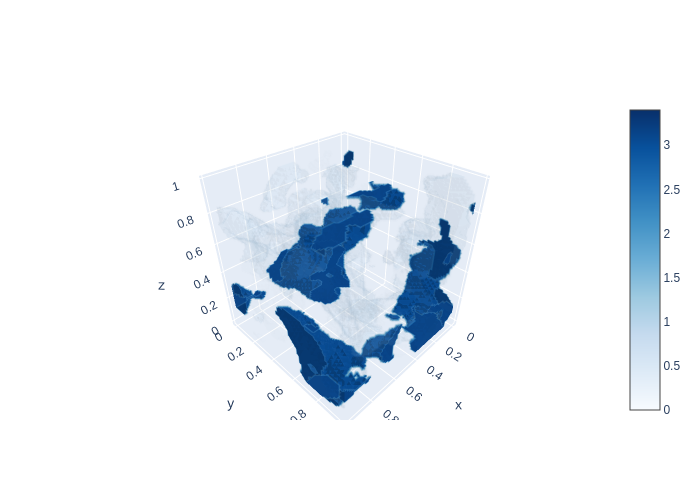} & \includegraphics[trim={6.1cm 2cm 7cm 4.5cm},scale=0.35,clip]{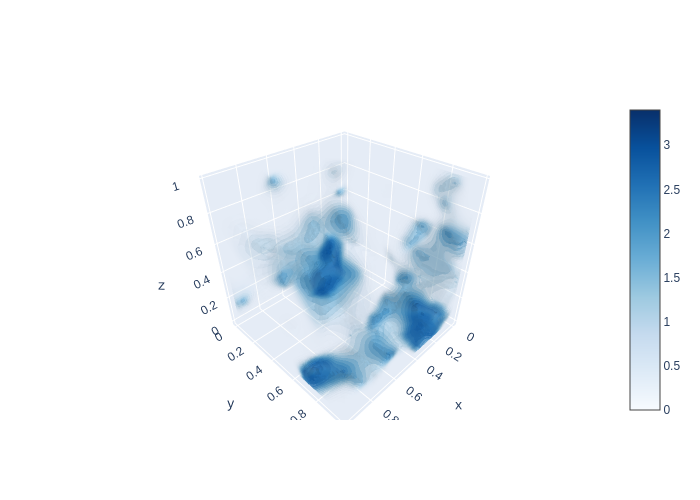} & \includegraphics[trim={6.1cm 2cm 7cm 4.5cm},scale=0.35,clip]{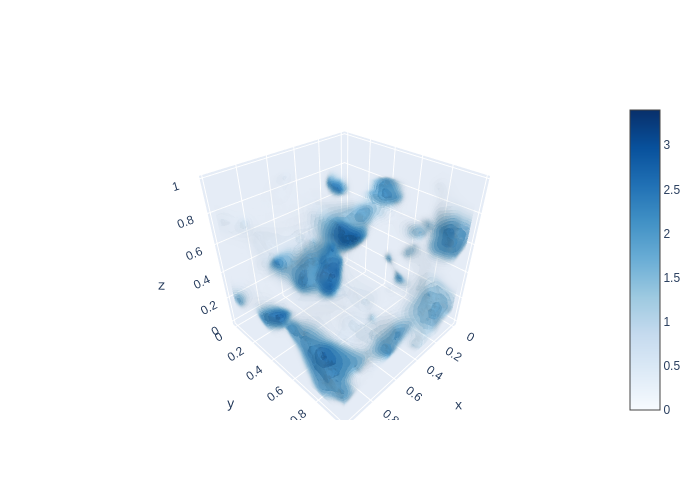} & \includegraphics[trim={6.1cm 2cm 7cm 4.5cm},scale=0.35,clip]{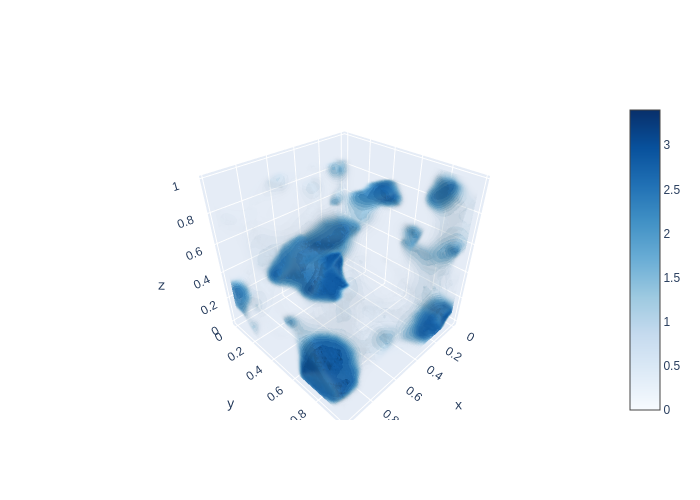} &\\
		& $\sf{8192}$ & $\sf{16384}$ & $\sf{32768}$\\
		& \includegraphics[trim={6.1cm 2cm 7cm 4.5cm},scale=0.35,clip]{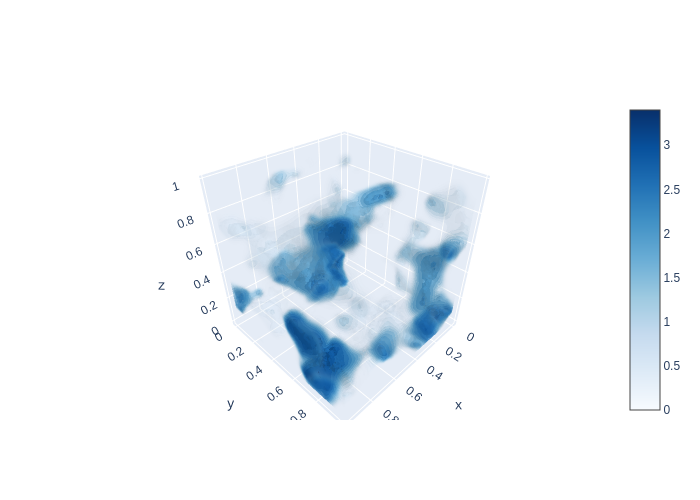} & \includegraphics[trim={6.1cm 2cm 7cm 4.5cm},scale=0.35,clip]{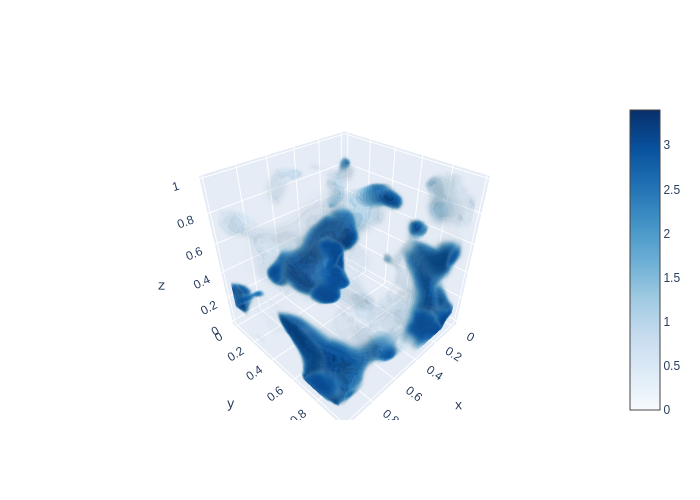} & \includegraphics[trim={6.1cm 2cm 7cm 4.5cm},scale=0.35,clip]{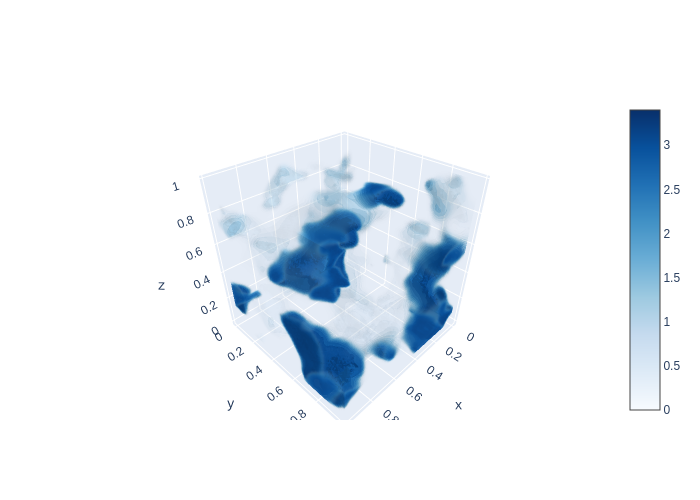} & \\
	\end{tabular}
	\caption{\textbf{3D Reconstructions.} More 3D reconstructions produced by $\mu$-Net-L at various dosages. The improvement in reconstruction quality as the dosage increases can be seen clearly}
	\label{fig:3d-reconstruction}
\end{figure}

\begin{figure}
	\centering
	\begin{tabular}{ccccccccccc}
		$\sf{Ground\;Truth}$ & $\sf{1024}$ & $\sf{2048}$ & $\sf{4096}$\\
		\includegraphics[trim={6.1cm 2cm 7cm 4.5cm},scale=0.35,clip]{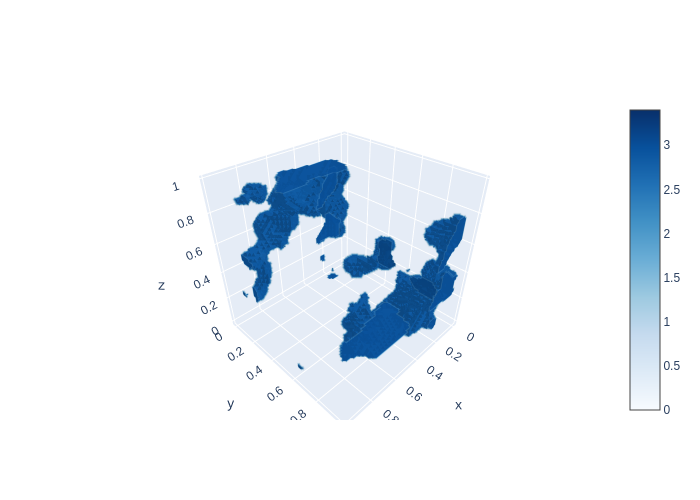} & \includegraphics[trim={6.1cm 2cm 7cm 4.5cm},scale=0.35,clip]{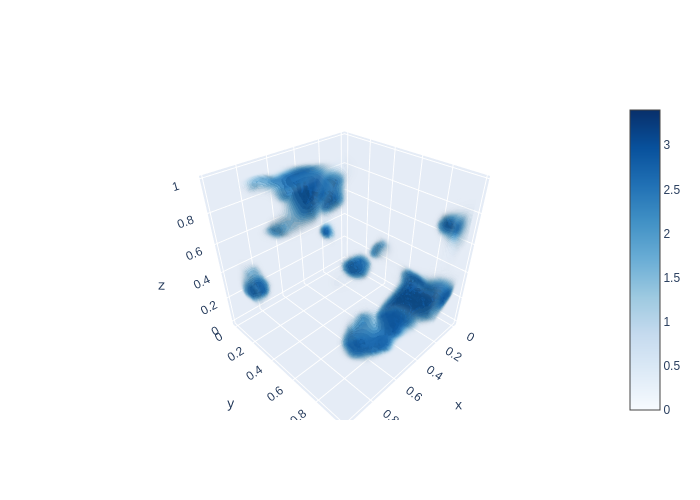} & \includegraphics[trim={6.1cm 2cm 7cm 4.5cm},scale=0.35,clip]{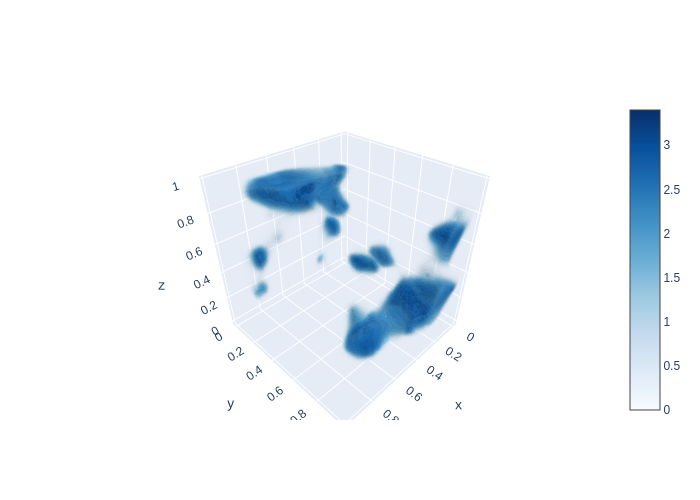} & \includegraphics[trim={6.1cm 2cm 7cm 4.5cm},scale=0.35,clip]{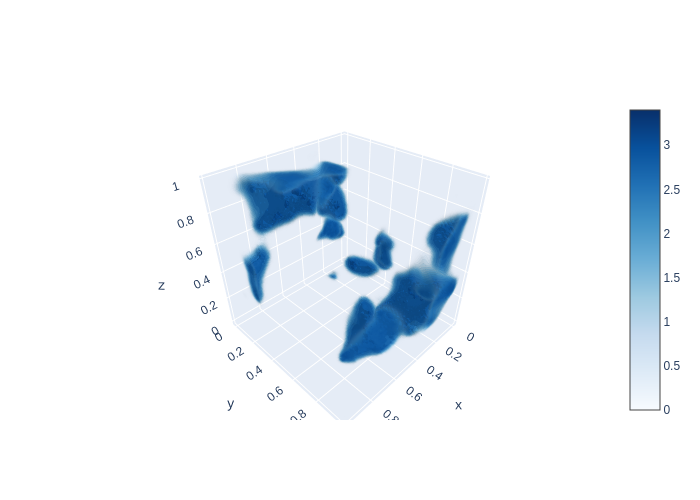} &\\
		& $\sf{8192}$ & $\sf{16384}$ & $\sf{32768}$\\
		& \includegraphics[trim={6.1cm 2cm 7cm 4.5cm},scale=0.35,clip]{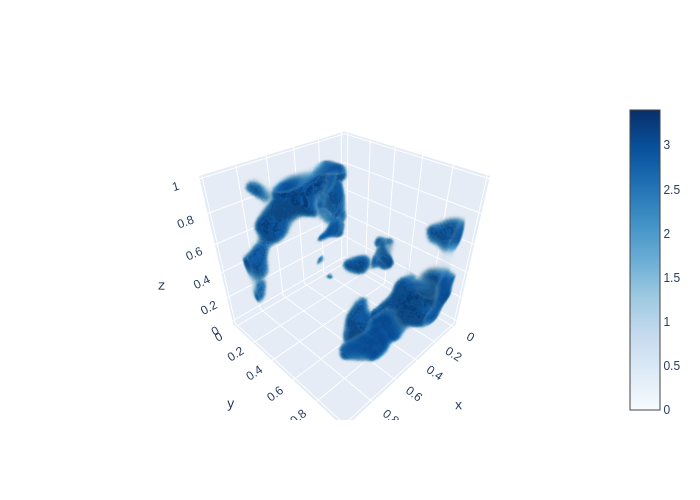} & \includegraphics[trim={6.1cm 2cm 7cm 4.5cm},scale=0.35,clip]{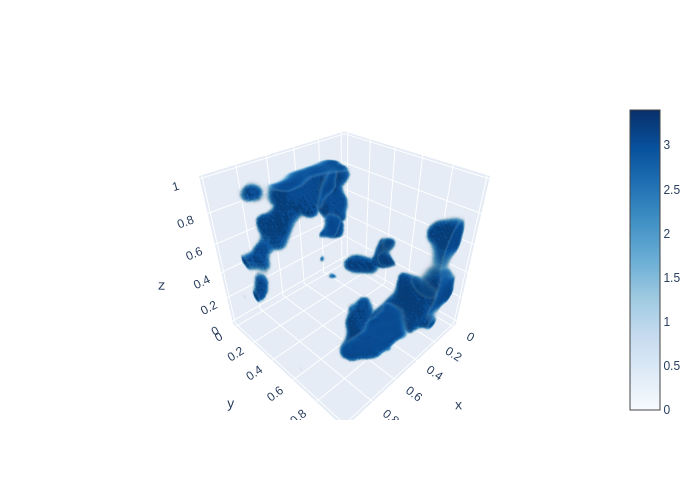} & \includegraphics[trim={6.1cm 2cm 7cm 4.5cm},scale=0.35,clip]{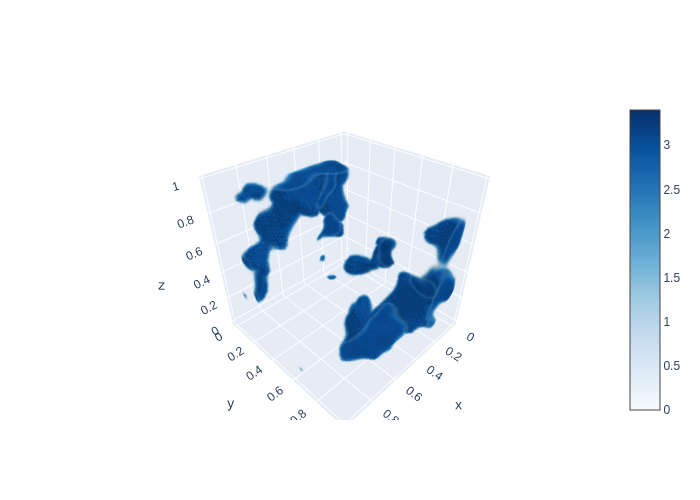} & \\
		\midrule\\
		$\sf{Ground\;Truth}$ & $\sf{1024}$ & $\sf{2048}$ & $\sf{4096}$\\
		\includegraphics[trim={6.1cm 2cm 7cm 4.5cm},scale=0.35,clip]{images/large/1024/3d_graph_7.png} & \includegraphics[trim={6.1cm 2cm 7cm 4.5cm},scale=0.35,clip]{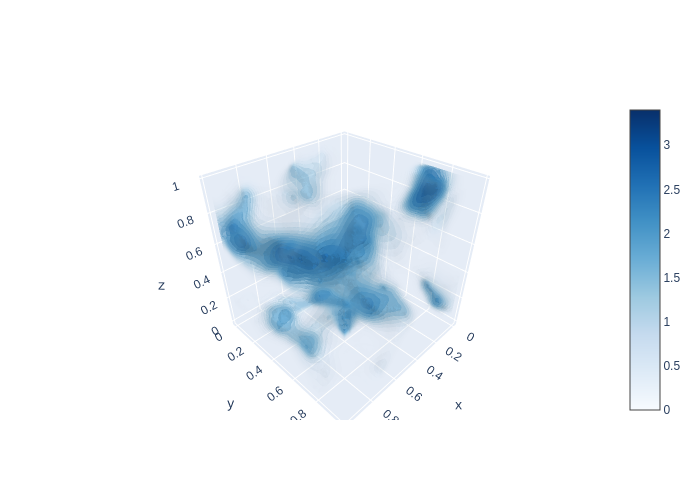} & \includegraphics[trim={6.1cm 2cm 7cm 4.5cm},scale=0.35,clip]{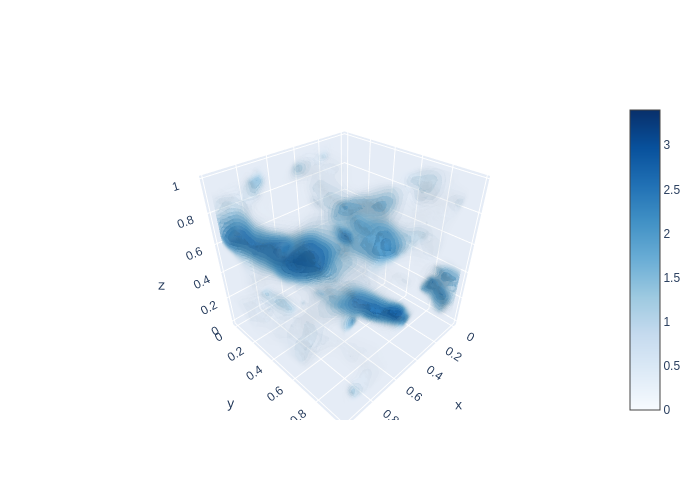} & \includegraphics[trim={6.1cm 2cm 7cm 4.5cm},scale=0.35,clip]{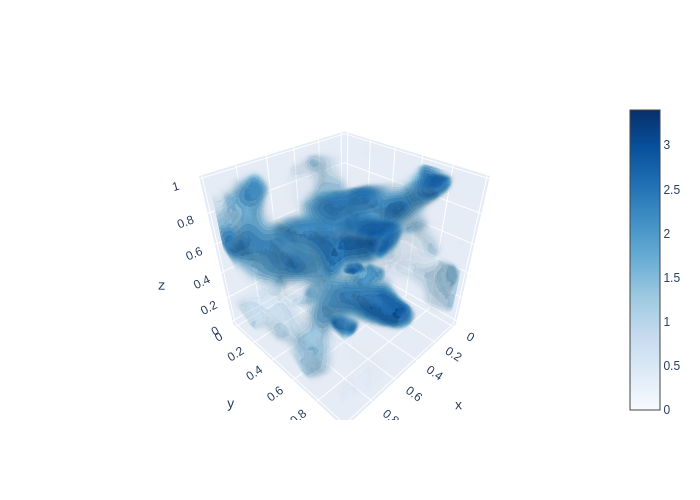} &\\
		& $\sf{8192}$ & $\sf{16384}$ & $\sf{32768}$\\
		& \includegraphics[trim={6.1cm 2cm 7cm 4.5cm},scale=0.35,clip]{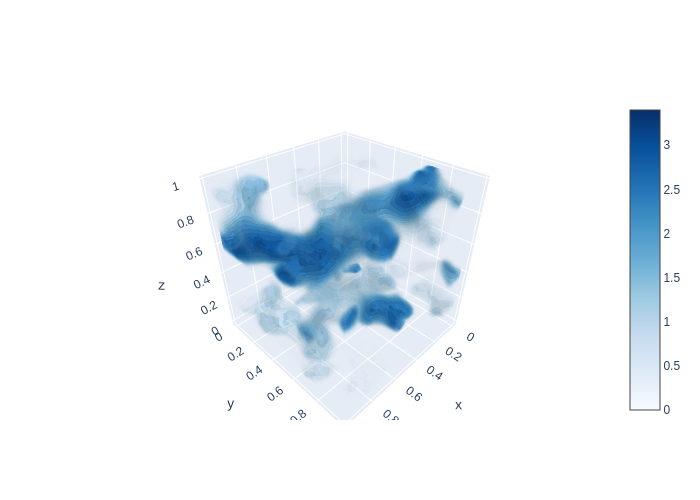} & \includegraphics[trim={6.1cm 2cm 7cm 4.5cm},scale=0.35,clip]{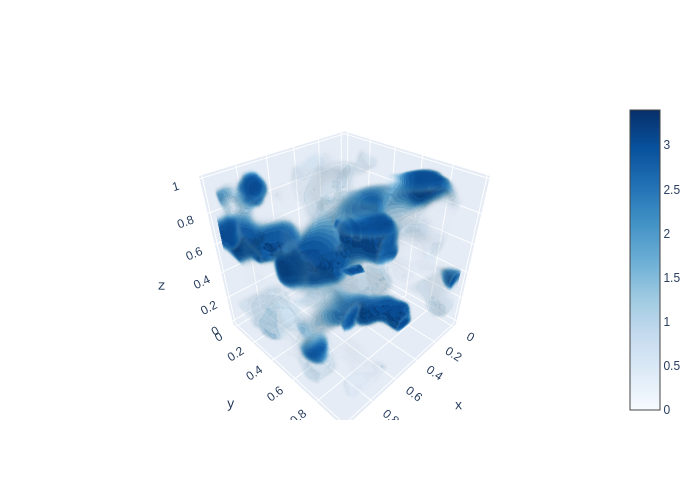} & \includegraphics[trim={6.1cm 2cm 7cm 4.5cm},scale=0.35,clip]{images/large/32768/3d_graph_7_pred.png} & \\
	\end{tabular}
	\caption{\textbf{3D Reconstructions (continued).} Even more 3D reconstructions produced by $\mu$-Net-L at various dosages. The improvement in reconstruction quality as the dosage increases can be seen clearly}
	\label{fig:3d-reconstruction-continued}
\end{figure}

\newpage

\section{Artifact Analysis}
\label{app:artifacts}

\paragraph{Little Squares.} We can see these little square in Figure \ref{fig:2d-crosssection} (first row). These are caused by the muons that do not scatter and are placed randomly along their trajectories. Since the model typically will see points placed within the voxels corresponding to there being actual material there, there is a slightly larger scattering density predicted at these regions where there should be nothing. These artifacts are only visible within the cross-section when there is nothing else inside (as is the case for the first row of Figure \ref{fig:2d-crosssection}).

\begin{figure}[h]
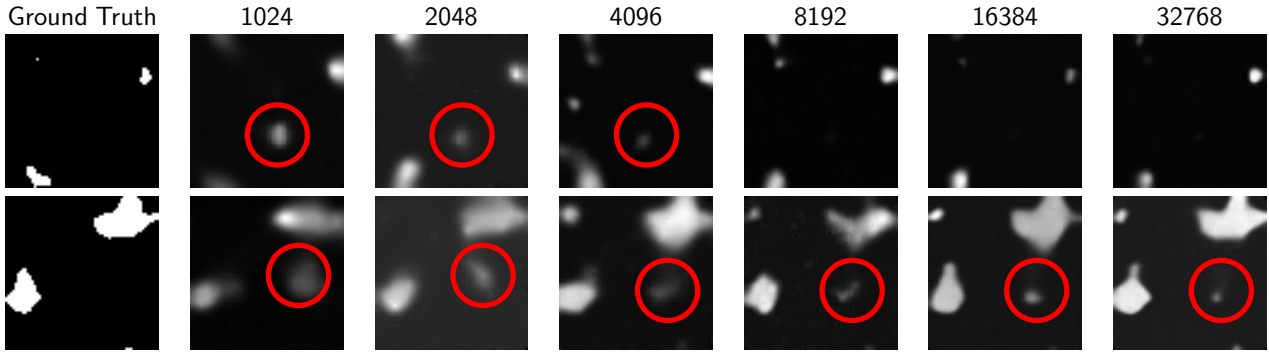

	\centering
	\begin{tabular}{ccccccccccc}
		$\sf{Ground\;Truth}$ & $\sf{1024}$ & $\sf{2048}$ & $\sf{4096}$ & $\sf{8192}$ & $\sf{16384}$ & $\sf{32768}$\\
		\includegraphics[scale=0.9]{images/large/1024/crosssection_5.png} & \stackinset{l}{7.2mm}{b}{2.6mm}{\textcolor{red}{\solidcirc[0]{2pt}{0.4}{0.4}}}{\includegraphics[scale=0.9]{images/large/1024/crosssection_5_pred.png}} & \stackinset{l}{7.2mm}{b}{2.6mm}{\textcolor{red}{\solidcirc[0]{2pt}{0.4}{0.4}}}{\includegraphics[scale=0.9]{images/large/2048/crosssection_5_pred.png}} & \stackinset{l}{7.2mm}{b}{2.6mm}{\textcolor{red}{\solidcirc[0]{2pt}{0.4}{0.4}}}{\includegraphics[scale=0.9]{images/large/4096/crosssection_5_pred.png}} & \includegraphics[scale=0.9]{images/large/8192/crosssection_5_pred.png} & \includegraphics[scale=0.9]{images/large/16384/crosssection_5_pred.png} & \includegraphics[scale=0.9]{images/large/32768/crosssection_5_pred.png} &\\
		\vspace{2mm}
		\includegraphics[scale=0.9]{images/large/1024/crosssection_6.png} & \stackinset{l}{10mm}{b}{5.5mm}{\textcolor{red}{\solidcirc[0]{2pt}{0.4}{0.4}}}{\includegraphics[scale=0.9]{images/large/1024/crosssection_6_pred.png}} & \stackinset{l}{10mm}{b}{5.5mm}{\textcolor{red}{\solidcirc[0]{2pt}{0.4}{0.4}}}{\includegraphics[scale=0.9]{images/large/2048/crosssection_6_pred.png}} & \stackinset{l}{10mm}{b}{3.5mm}{\textcolor{red}{\solidcirc[0]{2pt}{0.4}{0.4}}}{\includegraphics[scale=0.9]{images/large/4096/crosssection_6_pred.png}} & \stackinset{l}{10mm}{b}{3.5mm}{\textcolor{red}{\solidcirc[0]{2pt}{0.4}{0.4}}}{\includegraphics[scale=0.9]{images/large/8192/crosssection_6_pred.png}} & \stackinset{l}{10mm}{b}{3.5mm}{\textcolor{red}{\solidcirc[0]{2pt}{0.4}{0.4}}}{\includegraphics[scale=0.9]{images/large/16384/crosssection_6_pred.png}} & \stackinset{l}{10mm}{b}{3.5mm}{\textcolor{red}{\solidcirc[0]{2pt}{0.4}{0.4}}}{\includegraphics[scale=0.9]{images/large/32768/crosssection_6_pred.png}} &
	\end{tabular}
	\caption{\textbf{Hotspots.} 2D cross-sections of the 3D reconstructions produced by $\mu$-Net-L at various dosages. The false PoCA hotspots are circled in red. We see that as the dosage increases, these hotposts fade in prominence.}
	\label{fig:hotspots}
\end{figure}

\paragraph{False Hotspots.} We can see false hotspots in Figure \ref{fig:2d-crosssection} (2nd and 3rd law row) and highlighted in Figure \ref{fig:hotspots}. These are the result of muons scattering twice when passing through materials. PoCA assumes that muons scatter only once. This means that if a muon scatters twice, its PoCA point will end up somewhere end the midpoint of its actual scattering points. We also notice that as the dosage increases, these hotspots tend to fade in prominence. This is likely because with a larger dosages, the model is better able to distinguish between real scattering points and these false hotspots.

\begin{figure}[h]
	\centering
	\begin{tabular}{c|cccccccccc}
		$\sf{Ground\;Truth}$ & $\sf{Reconstructions}$\\
		\vspace{2mm}
		\includegraphics{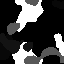} &
		\includegraphics{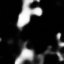} &
		\includegraphics{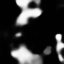} &
		\includegraphics{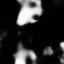} &
		\includegraphics{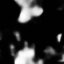} &
		\includegraphics{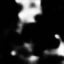} &\\
	\includegraphics[trim={6.1cm 2cm 7.2cm 4.5cm},scale=0.2,clip]{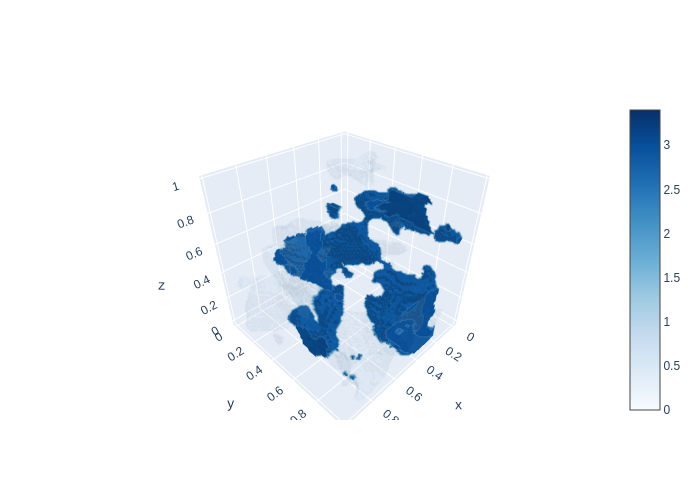} & \includegraphics[trim={6.1cm 2cm 7.2cm 4.5cm},scale=0.2,clip]{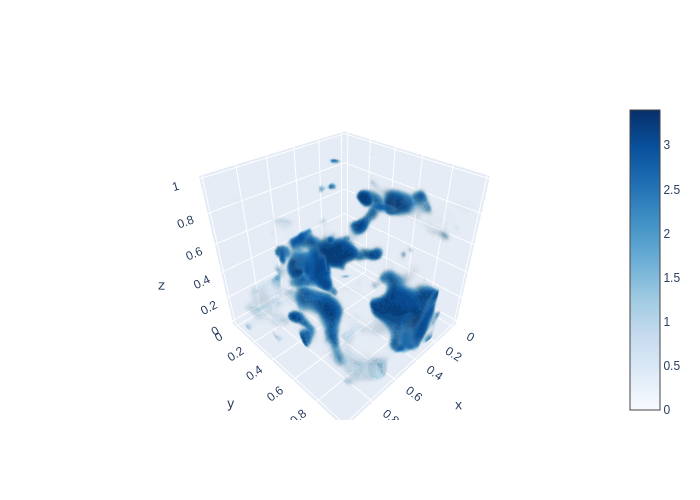} & \includegraphics[trim={6.1cm 2cm 7.2cm 4.5cm},scale=0.2,clip]{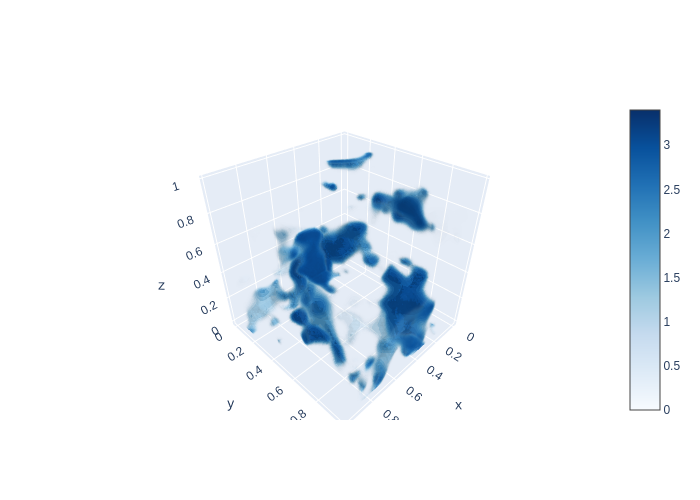} & \includegraphics[trim={6.1cm 2cm 7.2cm 4.5cm},scale=0.2,clip]{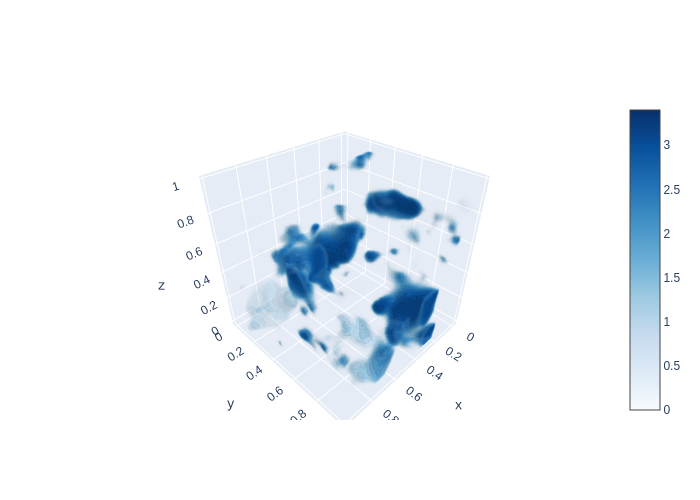} &  \includegraphics[trim={6.1cm 2cm 7.2cm 4.5cm},scale=0.2,clip]{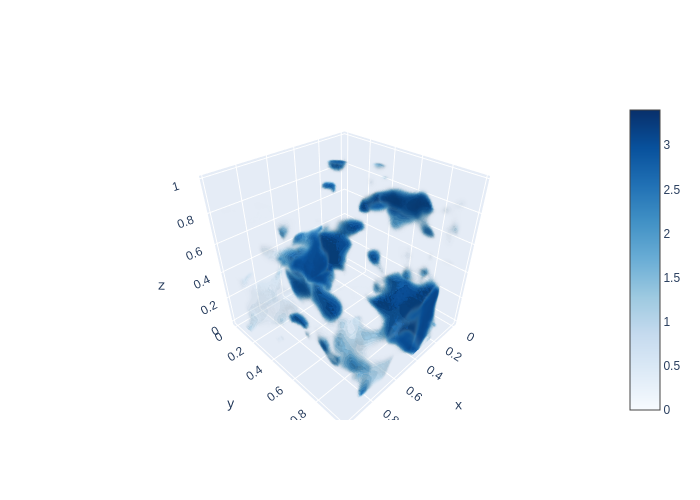} & \includegraphics[trim={6.1cm 2cm 7.2cm 4.5cm},scale=0.2,clip]{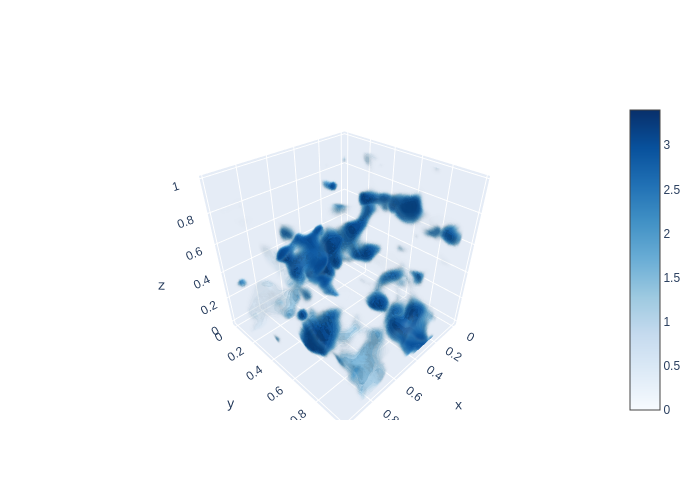} & 
	\end{tabular}
	\caption{\textbf{Distortions.} Several sample reconstructions of the same target object with a dosage of 8192. Each reconstruction uses a different sample of 8192 muons from the distribution of muon detections.}
	\label{fig:distortion}
\end{figure}

\paragraph{Distortions and Blurring.} In Figure \ref{fig:distortion}, many of the objects in the cross-section are visibly distorted, with these distortions shifting as the dosage increases (particularly noticeable in last row and 2nd row). This is due to the fundamentally random nature of this type of tomography and the low dosages of muons. Some scattering points will inevitably fall outside of what is the actual material. The model will then assume that these are part of the actual material, leading to blurring and distortions due to the fundamentally random nature of where these points will be found.

\newpage

\section{Raw Results}

\begin{table}[h]
	\caption{\textbf{Model Scaling.} The results of the model at different dosages for various sizes. The inference times are evaluated on 2 T4 GPUs with a batch size of 8. The best results are bolded. For PoCA, the inference times are evaluated on a P100 GPU.}
	\label{tab:model-scaling}
	\centering
	\vskip 0.15in
	\begin{tabular}{llllll}
		\toprule
		\textbf{Model} & \textbf{Dosage} & \textbf{Time}$\downarrow$ &  \textbf{MSE}$\downarrow$ & \textbf{MAE}$\downarrow$ & \textbf{PSNR}$\uparrow$ \\
		\midrule
		$\mu$-Net-T & 1024 & \textbf{126 ms} & \textbf{0.2276} & 0.2204 & \textbf{17.1426} \\
		$\mu$-Net-B & 1024 & 200 ms & 0.2318 & 0.2313 & 17.0255 \\
		$\mu$-Net-L & 1024 & 288 ms & 0.2295 & \textbf{0.2178} & 17.0646 \\
		PoCA & 1024 & 22.5s & 0.4595 & 0.2447 & 13.6627 \\
		\midrule
		$\mu$-Net-T & 2048 & \textbf{135 ms} & 0.1965 & 0.1989 & 17.7786 \\
		$\mu$-Net-B & 2048 & 208 ms & 0.1936 & \textbf{0.1911} & 17.8486 \\
		$\mu$-Net-L & 2048 & 306 ms & \textbf{0.1929} & 0.1918 & \textbf{17.8633} \\
		PoCA & 2048 & 43.8s & 0.4338 & 0.2465 & 13.9112\\
		\midrule
		$\mu$-Net-T & 4096 & \textbf{141 ms} & 0.1653 & 0.1725 & 18.5347 \\
		$\mu$-Net-B & 4096 & 218 ms & 0.1649 & 0.1738 & \textbf{18.5504} \\
		$\mu$-Net-L & 4096 & 301 ms & \textbf{0.1644} & \textbf{0.1697} & 18.5388 \\
		PoCA & 4096 & 79.8s & 0.3950 & 0.2466 & 14.3228 \\
		\midrule
		$\mu$-Net-T & 8192 & \textbf{169 ms} & 0.1388 & 0.1438 & 19.2979 \\
		$\mu$-Net-B & 8192 & 234 ms & 0.1350 & 0.1457 & 19.3958 \\
		$\mu$-Net-L & 8192 & 325 ms & \textbf{0.1348} & \textbf{0.1389} & \textbf{19.4232} \\
		PoCA & 8192 & 164s & 0.3660 & 0.2420 & 15.0769 \\
		\midrule
		$\mu$-Net-T & 16384 & \textbf{246 ms} & 0.1169 & 0.1207 & 20.0433 \\
		$\mu$-Net-B & 16384 & 293 ms & 0.1118 & 0.1238 & 20.2685 \\
		$\mu$-Net-L & 16384 & 384 ms & \textbf{0.1062} & \textbf{0.1180} & \textbf{20.4322} \\
		PoCA & 16384 & 310s & 0.3285 & 0.2315 & 15.5586 \\
		\midrule
		$\mu$-Net-T & 32768 & \textbf{347 ms} & 0.0993 & 0.1156 & 20.7906 \\
		$\mu$-Net-B & 32768 & 434 ms & 0.0919 & 0.1040 & 21.1595 \\
		$\mu$-Net-L & 32768 & 538 ms & \textbf{0.0875} & \textbf{0.0983} & \textbf{21.3530} \\
		PoCA & 32768 & 612s & 0.3092 & 0.2258 & 17.1091 \\
		\bottomrule
	\end{tabular}
\end{table}

\begin{table}
	\caption{\textbf{Detector Resolutions.} Results of various methods for resolutions of the detector. $\mu$-Net-T$^*$ indicates that the model was finetuned on the new data for 10 epochs. The best results are bolded.}
	\label{tab:detector-resolution}
	\vskip 0.15in
	\centering
	\begin{tabular}{llllll}
		\toprule
		\textbf{Model} & \textbf{Detector Resolution} & \textbf{MSE}$\downarrow$ & \textbf{MAE}$\downarrow$ & \textbf{PSNR}$\uparrow$ \\
		\midrule
		$\mu$-Net-T & $64\times64$ & 0.2545 & 0.2647 & 16.5660 \\
		$\mu$-Ne-Tt$^*$ & $64\times64$ & \textbf{0.2317} & \textbf{0.2107} & \textbf{17.0075} \\
		PoCA & $64\times64$ & 0.5210 & 0.2690 & 13.1784 \\
		\midrule
		$\mu$-Net-T & $128\times128$ & 0.2259 & 0.2383 & 17.1177 \\
		$\mu$-Net-T$^*$ & $128\times128$ & \textbf{0.2174} & \textbf{0.2349} & \textbf{17.3046} \\
		PoCA & $128\times128$ & 0.5226 & 0.2696 & 13.1625 \\
		\midrule
		$\mu$-Net-T & $256\times256$ & 0.2123 & \textbf{0.2221} & 17.4097 \\
		$\mu$-Net-T$^*$ & $256\times256$ & \textbf{0.2068} & 0.2226 & \textbf{17.5387} \\
		PoCA & $256\times256$ & 0.5210 & 0.2690 & 13.1784 \\
		\midrule
		$\mu$-Net-T & $1024\times1024$ & \textbf{0.1991} & \textbf{0.2057} & \textbf{17.7123} \\
		$\mu$-Net-T$^*$ & $1024\times1024$ & 0.2035 & 0.2119 & 17.6165 \\
		PoCA & $1024\times1024$ & 0.5210 & 0.2690 & 13.1784 \\
		\midrule
		$\mu$-Net-T & $2048\times2048$ & \textbf{0.1970} & 0.2023 & \textbf{17.7616} \\
		$\mu$-Net-T$^*$ & $2048\times2048$ & 0.2013 & \textbf{0.1901} & 17.6695 \\
		PoCA & $2048\times2048$ & 0.5210 & 0.2690 & 13.1784 \\
		\midrule
		$\mu$-Net-T & $\infty$ & \textbf{0.1936} & \textbf{0.1911} & \textbf{17.8486}\\
		PoCA & $\infty$ &  0.4338 & 0.2465 & 13.9112 \\
		\bottomrule
	\end{tabular}
\end{table}

\begin{table}
	\caption{\textbf{Momentum Error.} Results of various models for different levels of error in the momentum estimate. $\mu$-Net$^*$ indicates that the model was finetuned on the new data for 10 epochs. The best results are bolded.}
	\label{tab:momentum-estimate}
	\vskip 0.15in
	\centering
	\begin{tabular}{llllll}
		\toprule
		\textbf{Model} & \textbf{$\Delta\mathbf{p}$} & \textbf{MSE}$\downarrow$ & \textbf{MAE}$\downarrow$ & \textbf{PSNR}$\uparrow$ \\
		\midrule
		$\mu$-Net & 0\% & 0.1951 & 0.1977 & 17.8110 \\
		$\mu$-Net$^*$ & 0\% & \textbf{0.1920} & \textbf{0.1938} & \textbf{17.8865} \\
		PoCA & 0\% & 0.4224 & 0.2442 & 14.0280 \\
		\midrule
		$\mu$-Net & 20\% & \textbf{0.1965} & \textbf{0.1989} & \textbf{17.7786} \\
		PoCA & 20\% & 0.4338 & 0.2465 & 13.9112 \\
		\midrule
		$\mu$-Net & 40\% & 0.2004 & 0.2033 & 17.6844 \\
		$\mu$-Net$^*$ & 40\% & \textbf{0.1987} & \textbf{0.1940} & \textbf{17.7258} \\
		PoCA & 40\% & 0.4577 & 0.2515 & 13.6717 \\
		\midrule
		$\mu$-Net & 60\% & 0.2230 & 0.2207 & 17.2291 \\
		$\mu$-Net$^*$ & 60\% & \textbf{0.1989} & \textbf{0.1920} & \textbf{17.7326} \\
		PoCA & 60\% & 0.5316 & 0.2616 & 13.1310 \\
		\midrule
		$\mu$-Net & 80\% & 0.2761 & 0.2558 & 16.2882 \\
		$\mu$-Net$^*$ & 80\% & \textbf{0.2020} & \textbf{0.2023} & \textbf{17.6497} \\
		PoCA & 80\% & 0.6228 & 0.2730 & 12.5438 \\
		\midrule
		$\mu$-Net & 100\% & 0.3293 & 0.2866 & 15.6005 \\
		$\mu$-Net$^*$ & 100\% & \textbf{0.2000} & \textbf{0.2014} & \textbf{17.7015} \\
		PoCA & 100\% & 0.8279 & 0.2933 & 11.5817 \\
		\bottomrule
	\end{tabular}
\end{table}

\end{document}